\newif\ificml
\icmlfalse %

\newif\ifanonymous
\anonymousfalse

\newif\ifarxiv

\ificml
    \arxivfalse
\else
    \arxivtrue
\fi

\documentclass[12pt]{article}
\usepackage[utf8]{inputenc}
\usepackage{graphicx}
\usepackage[dvipsnames]{xcolor}
\usepackage[colorlinks=false]{hyperref}
\usepackage{amsmath,amsthm,amssymb,fullpage,comment}
\usepackage{algorithm,algorithmic}
\newcommand{\ex}[2]{{\ifx&#1& \mathbb{E} \else
\underset{#1}{\mathbb{E}} \fi \left[#2\right]}}
\newcommand{\pr}[2]{{\ifx&#1& \mathbb{P} \else
\underset{#1}{\mathbb{P}} \fi \left[#2\right]}}
\newcommand{\var}[2]{{\ifx&#1& \mathsf{Var} \else
\underset{#1}{\mathsf{Var}} \fi \left[#2\right]}}
\newcommand{\dr}[3]{\mathrm{D}_{#1}\left(#2\middle\|#3\right)}
\newcommand{\nope}[1]{}
\newtheorem{lem}{Lemma}
\newtheorem{defn}[lem]{Definition}
\newtheorem{cor}[lem]{Corollary}
\newtheorem{prop}[lem]{Proposition}
\newtheorem{thm}[lem]{Theorem}
\newtheorem{rem}[lem]{Remark}

\usepackage[style=alphabetic,backend=bibtex,maxalphanames=10,maxbibnames=20,maxcitenames=10,giveninits=true,doi=false,url=true]{biblatex}
\newcommand*{\citet}[1]{\AtNextCite{\AtEachCitekey{\defcounter{maxnames}{10}}}
\textcite{#1}}
\newcommand*{\citetall}[1]{\AtNextCite{\AtEachCitekey{\defcounter{maxnames}{999}}}
\textcite{#1}}
\newcommand*{\citep}[1]{\cite{#1}}

\newcommand{\tnote}[1]{\marginpar{{\footnotesize{T: #1}}}}

\title{The Distributed Discrete Gaussian Mechanism for Federated Learning with Secure Aggregation}
\ifanonymous
\author{Anonymous Authors}
\else
\author{
Peter Kairouz\thanks{Google Research ~\dotfill~\texttt{kairouz@google.com}} 
\and 
Ziyu Liu\thanks{Google Research~\dotfill~\texttt{klz@google.com}} 
\and 
Thomas Steinke\thanks{Google Research, Brain Team~\dotfill~\texttt{ddg@thomas-steinke.net}}
}
\fi

\date{}

\begin{document}
\maketitle

\footnotetext{Accepted for publication at \href{https://icml.cc/Conferences/2021}{the 38th International Conference on Machine Learning (ICML 2021)}.}

\begin{abstract}
We consider training models on private data that are distributed across user devices. To ensure privacy, we add on-device noise and use secure aggregation so that only the noisy sum is revealed to the server. We present a comprehensive end-to-end system, which appropriately discretizes the data and adds discrete Gaussian noise before performing secure aggregation. We provide a novel privacy analysis for sums of discrete Gaussians and carefully analyze the effects of data quantization and modular summation arithmetic. Our theoretical guarantees highlight the complex tension between communication, privacy, and accuracy. Our extensive experimental results demonstrate that our solution is essentially able to match the accuracy to central differential privacy with less than 16 bits of precision per value.
\end{abstract}

\tableofcontents

\section{Introduction}
Software and service providers rely on increasingly complex data analytics and machine learning models to improve their services. However, training these machine learning models hinges on the availability of large datasets, which are often distributed across user devices and contain sensitive information. The collection of these datasets comes with several privacy risks -- can the service provider address issues around consent, transparency, control, breaches, persistence, processing, and release of data? There is thus a strong desire for technologies which systematically address privacy concerns while preserving, to the best extent possible, the utility of the offered services. 

To address this need, several privacy-enhancing technologies have been studied and built over the past few years. Prominent examples of such technologies include federated learning (FL) to ensure that raw data never leaves users' devices~\cite{mcmahan2017communication,kairouz2019advances}, cryptographic secure aggregation (SecAgg) to prevent a server from inspecting individual user updates~\cite{bonawitz2017practical,bell2020secagg}, and differentially private stochastic gradient descent (DP-SGD) to train models with provably limited information leakage~\cite{abadi2016deep,tramer2020differentially}. While these technologies have been extremely well studied in a separate fashion, little work has focused on understanding precisely how they can be combined in a rigorous and principled fashion.
Towards this end, we present a comprehensive end-to-end system where each client appropriately discretizes their model update and adds discrete Gaussian noise to it before sending it for modular secure summation using SecAgg.
This provides the first concrete step towards building a communication-efficient FL system with distributed DP\footnote{See ``Distributed DP'' paragraph in Section \ref{sec:related_work} for a definition of this notion of DP and a literature review.} and SecAgg guarantees.  

\ificml
\textbf{Organization}\quad
The remainder of the paper is organized as follows.
We present some preliminaries in Section~\ref{sec:prelim}, summarize our main results in Section~\ref{sec:main-results}, and review related works in Section~\ref{sec:related_work}. In Section~\ref{sec:dg}, we introduce the distributed discrete Gaussian mechanism, analyze its privacy guarantees, and show how to apply it in federated learning. We present experiments in Section~\ref{sec:experiments} and conclude with a few interesting extensions in Section~\ref{sec:conclusion}. All formal definitions, proofs, algorithmic details, extensions, and additional experiments are deferred to the supplementary material.

\section{Preliminaries}\label{sec:prelim}
\ifarxiv
We begin by defining the R\'enyi divergences, which we use throughout to quantify privacy.

\begin{defn}[R\'enyi divergences]
Let $P$ and $Q$ be probability distributions on some common domain $\Omega$. Assume that $P$ is absolutely continuous with respect to $Q$ so that the Radon-Nikodym derivative $P(x)/Q(x)$ is well-defined for
$x \in \Omega$.\footnote{If $P$ is \emph{not} absolutely continuous with respect to $Q$, then we define all of these divergences to be infinity. The Radon-Nikodym derivative is only unique up to measure zero events. In the definition of $\dr{\infty}{P}{Q}$ and $\dr{\pm\infty}{P}{Q}$, we ignore zero probability events (i.e., we take the essential supremum). That is, we assume the Radon-Nikodym derivative is chosen to minimize these quantities.}

For $\alpha \in (1,\infty)$, we define the R\'enyi divergence of order $\alpha$ of $P$ with respect to $Q$ as \begin{equation}\dr{\alpha}{P}{Q} := \frac{1}{\alpha-1} \log \ex{X \gets P}{\left(\frac{P(X)}{Q(X)}\right)^{\alpha-1}}.\end{equation}
We also define
\begin{align}
    \dr{1}{P}{Q} &:= \ex{X \gets P}{\log\left(\frac{P(X)}{Q(X)}\right)} = \lim_{\alpha \to 1} \dr{\alpha}{P}{Q},\\
    \dr{\infty}{P}{Q} &:= \sup_{x \in \Omega} {\log\left(\frac{P(x)}{Q(x)}\right)} = \lim_{\alpha \to \infty} \dr{\alpha}{P}{Q},\\
    \dr{\pm\infty}{P}{Q} &:= \sup_{x \in \Omega} \left|\log\left(\frac{P(x)}{Q(x)}\right)\right| = \max \{ \dr{\infty}{P}{Q}, \dr{\infty}{Q}{P} \},\\
    \dr{*}{P}{Q} &:= \sup_{\alpha \in (1,\infty)} \frac{1}{\alpha} \dr{\alpha}{P}{Q}.
\end{align}
\end{defn}
We will abuse this notation by considering the divergence between random variables when we mean the divergence between their respective distributions.

We now state some properties of the R\'enyi divergences; proofs and further properties can be found in the literature \cite{BunS16,BunS19}. 

\begin{lem}\label{lem:renyi-misc}
Let $P,Q,R$ be probability distributions such that $P$ is absolutely continuous with respect to $Q$ and $Q$ is absolutely continuous with respect to $R$. 
\ifarxiv
Then the following hold.
\begin{itemize}
    \item \textbf{Gaussian divergence:} For all $\mu,\mu'\in\mathbb{R}$ and all $\sigma>0$, $\dr{*}{\mathcal{N}(\mu,\sigma^2)}{\mathcal{N}(\mu',\sigma^2)} = \frac{(\mu-\mu')^2}{2\sigma^2}$.
    \item \textbf{Conversion from max divergence:} $\dr{*}{P}{Q} \le \min \{ \dr{\infty}{P}{Q}, \frac12 \left( \dr{\pm\infty}{P}{Q} \right)^2 \}$.
    \item \textbf{Triangle inequality:} $\dr{*}{P}{R} \le \left( \sqrt{\dr{*}{P}{Q}} + \sqrt{\dr{*}{Q}{R}} \right)^2$ and $\dr{\alpha}{P}{R} \le \min\{ \dr{\alpha}{P}{Q} + \dr{\infty}{Q}{R} , \dr{\infty}{P}{Q} + \dr{\alpha}{Q}{R} \}$  for all $\alpha \in [1,\infty] \cup \{*\}$.
    \item \textbf{Product distributions (non-adaptive composition):} If $P=P_1 \times P_2$ is a product distribution and $Q=Q_1 \times Q_2$ is a corresponding product distriution, then $\dr{\alpha}{P}{Q} = \dr{\alpha}{P_1}{Q_1} + \dr{\alpha}{P_2}{Q_2}$ for all $\alpha \in [1,\infty] \cup \{*\}$.
    \item \textbf{Postprocessing (data processing inequality):} $0 \le \dr{\alpha}{f(P)}{f(Q)} \le \dr{\alpha}{P}{Q}$ for all $\alpha \in [1,\infty] \cup \{*\}$ and all $f$, where $f(P)$ denotes the distribution obtained by applying some function $f$ to a sample from the distribution $P$. This also holds if $f$ is an independently randomized function.
    \item \textbf{Monotonicity:} $\dr{\alpha}{P}{Q} \le \dr{\alpha'}{P}{Q}$ whenever $1 \le \alpha \le \alpha' \le \infty$.
    \item \textbf{(Quasi)convexity:} If $P'$ is a distribution on the same space as $P$ and $Q'$ is a distribution on the same space as $Q$ and $P'$ is absolutely continuous with respect to $Q'$, then \[\dr{1}{tP+(1-t)P'}{tQ+(1-t)Q'} \le t\cdot \dr{1}{P}{Q} + (1-t) \cdot \dr{1}{P'}{Q'}\] and, for $\alpha \in (1,\infty)$, \begin{align*}
        \dr{\alpha}{tP+(1-t)P'}{tQ+(1-t)Q'} &\le \frac{\log \left( t\cdot e^{(\alpha-1)\dr{\alpha}{P}{Q}} + (1-t) \cdot e^{(\alpha-1)\dr{\alpha}{P'}{Q'}}\right)}{\alpha-1} \\
        &\le \max\{ \dr{\alpha}{P}{Q} , \dr{\alpha}{P'}{Q'} \},
    \end{align*}
    where $tP+(1-t)P'$ denotes the convex combination of distributions.
\end{itemize}
\else
Then, we have

\textbf{Gaussian divergence:} For all $\mu,\mu'\in\mathbb{R}$ and all $\sigma>0$, $\dr{*}{\mathcal{N}(\mu,\sigma^2)}{\mathcal{N}(\mu',\sigma^2)} = \frac{(\mu-\mu')^2}{2\sigma^2}$.

\textbf{Conversion from max divergence:}\\
$\dr{*}{P}{Q} \le \min \{ \dr{\infty}{P}{Q}, \frac12 \left( \dr{\pm\infty}{P}{Q} \right)^2 \}$.

\textbf{Triangle inequality:}\\ 
$\dr{*}{P}{R} \le \left( \sqrt{\dr{*}{P}{Q}} + \sqrt{\dr{*}{Q}{R}} \right)^2$ and $\dr{\alpha}{P}{R} \le \min\{ \dr{\alpha}{P}{Q} + \dr{\infty}{Q}{R} , \dr{\infty}{P}{Q} + \dr{\alpha}{Q}{R} \}$  for all $\alpha \in [1,\infty] \cup \{*\}$.

\textbf{Product distributions (non-adaptive composition):}\\ 
If $P=P_1 \times P_2$ is a product distribution and $Q=Q_1 \times Q_2$ is a corresponding product distriution, then $\dr{\alpha}{P}{Q} = \dr{\alpha}{P_1}{Q_1} + \dr{\alpha}{P_2}{Q_2}$ for all $\alpha \in [1,\infty] \cup \{*\}$.

\textbf{Postprocessing (data processing inequality):} \\ $0 \le \dr{\alpha}{f(P)}{f(Q)} \le \dr{\alpha}{P}{Q}$ for all $\alpha \in [1,\infty] \cup \{*\}$ and all $f$, where $f(P)$ denotes the distribution obtained by applying some function $f$ to a sample from the distribution $P$. $f$ can be an independently randomized function.

\textbf{Monotonicity:} $\dr{\alpha}{P}{Q} \le \dr{\alpha'}{P}{Q}$ whenever $1 \le \alpha \le \alpha' \le \infty$.
\fi
\end{lem}

\fi

\ifarxiv
Now we can state the definitions of concentrated differential privacy \cite{dwork2016concentrated,BunS16} and R\'enyi differential privacy \cite{mironov2017renyi} and relate these to the standard definition of differential privacy \cite{DworkMNS06,DworkKMMN06}. 
We adopt \emph{user-level privacy} -- i.e., each entry in the input corresponds to \emph{all} the records associated with a single person \citep{mcmahan2018learning}. Thus the differential privacy distributional similarity guarantee holds with respect to adding or removing all of the data belonging to a single person. This is stronger than the commonly-used notion of item level privacy where, if a user contributes multiple records, only the addition or removal of one record is protected. 

We choose to define differential privacy with respect to adding or removing the records of an individual, rather than replacing the records. Since replacement can be achieved by a combination of an addition and a removal, group privacy (a.k.a.~the triangle inequality) implies a differential privacy guarantee for replacement; however, the privacy parameter will be doubled.
\else
We begin by defining the R\'enyi divergence of order $\alpha \in (1,\infty)$ of distribution $P$ with respect to distribution $Q$ as \[\dr{\alpha}{P}{Q} := \frac{1}{\alpha-1} \log \ex{X \gets P}{\left(\frac{P(X)}{Q(X)}\right)^{\alpha-1}}.\]
We now state the definitions of concentrated DP~\cite{BunS16} and R\'enyi DP~\cite{mironov2017renyi} and relate these to the standard definition of differential privacy~\cite{DworkMNS06,DworkKMMN06}. 
\fi
\ifarxiv
We define $\mathcal{X}^* = \bigcup_{n=0}^\infty \mathcal{X}^n$ to be the set of varying-size inputs from $\mathcal{X}$.
\fi

Concentrated differential privacy is a version of differential privacy that captures many natural techniques for attaining differential privacy and gives sharp composition results, among other features. We use the definition of\citet{BunS16} which is a simplification of the original definition of\citet{dwork2016concentrated}.\footnote{Specifically, we use what is also known as zero-concentrated differential privacy (zCDP) \cite{BunS16}, although we drop the ``zero'' qualifier for brevity, as there is no need to distinguish it from the original version of concentrated differential privacy \cite{dwork2016concentrated}. The two versions are loosely equivalent, but the version we use has cleaner mathematical properties.}

\begin{defn}[Concentrated Differential Privacy \cite{dwork2016concentrated,BunS16}]
A randomized algorithm \ifarxiv $M : \mathcal{X}^* \to \mathcal{Y}$ \else $M : \mathcal{X} \to \mathcal{Y}$ \fi  satisfies $\frac12\varepsilon^2$-concentrated differential privacy iff, for all \ifarxiv $x,x'\in \mathcal{X}^*$ \else $x,x'\in \mathcal{X}$ \fi differing by the addition or removal of a single user's records, we have \ifarxiv $\dr{*}{M(x)}{M(x')} \le \frac12 \varepsilon^2$ \else $\sup_{\alpha \in (1,\infty)} \frac{1}{\alpha} \dr{\alpha}{M(x)}{M(x')} \le \frac12 \varepsilon^2$ \fi .
\end{defn}

A more general relaxation of concentrated differential privacy is R\'enyi differential privacy, which was defined by\citet{mironov2017renyi}.

\begin{defn}[R\'enyi Differential Privacy \cite{mironov2017renyi}]
A randomized algorithm \ifarxiv $M : \mathcal{X}^* \to \mathcal{Y}$ \else $M : \mathcal{X} \to \mathcal{Y}$ \fi  satisfies $(\alpha,\varepsilon)$-R\'enyi differential privacy iff, for all \ifarxiv $x,x'\in \mathcal{X}^*$ \else $x,x'\in \mathcal{X}$ \fi  differing by the addition or removal of a single user's records, we have $\dr{\alpha}{M(x)}{M(x')} \le \varepsilon$.
\end{defn}

We also use the original version of differential privacy that was defined by\citetall{DworkMNS06} with $\delta=0$ (a.k.a.~pure or pointwise differential privacy) and with the possibility of $\delta>0$ (i.e., approximate differential privacy) by\citetall{DworkKMMN06}.

\begin{defn}[Differential Privacy \cite{DworkMNS06,DworkKMMN06}]
A randomized algorithm \ifarxiv $M : \mathcal{X}^* \to \mathcal{Y}$ \else $M : \mathcal{X} \to \mathcal{Y}$ \fi  satisfies $(\varepsilon,\delta)$-differential privacy iff, for all \ifarxiv $x,x'\in \mathcal{X}^*$ \else $x,x'\in \mathcal{X}$ \fi  differing by the addition or removal of a single user's records, we have \begin{equation}\pr{}{M(x) \in E} \le e^\varepsilon \cdot \pr{}{M(x') \in E} + \delta\end{equation} for all events $E \subset \mathcal{Y}$.
We refer to $(\varepsilon,0)$-differential privacy as pure differential privacy or pointwise differential privacy and we refer to $(\varepsilon,\delta)$-differential privacy with $\delta>0$ as approximate differential privacy.
\end{defn}
We remark that $(\varepsilon,0)$-DP is equivalent to $(\infty,\varepsilon)$-R\'enyi DP. Similarly, $\frac12\varepsilon^2$-concentrated DP is equivalent to satisfying $(\alpha,\frac12\varepsilon^2\alpha)$-R\'enyi DP simultaneously for all $\alpha \in (1,\infty)$.

In addition we have the following conversion lemma \cite{BalleBGHS20,CanonneKS20,AsoodehLCKS20} from concentrated DP to approximate DP.

\begin{lem}
If $M$ satisfies $(\varepsilon,0)$-differential privacy, then it satisfies $\frac12\varepsilon^2$-concentrated differential privacy.
If $M$ satisfies $\frac12 \varepsilon^2$-concentrated differential privacy, then, for any $\delta>0$, $M$ satisfies $(\varepsilon_\text{aDP}(\delta),\delta)$-differential privacy, where
\ificml
\begin{align*}
\varepsilon_\text{aDP}(\delta) &= \inf_{\alpha>1} \frac12 \varepsilon^2 \alpha + \frac{\log(1/\alpha\delta)}{\alpha-1} + \log(1-1/\alpha) &&\\ 
& \le  \varepsilon \cdot \left( \sqrt{2\log(1/\delta)} + \varepsilon/2 \right). &&
\end{align*}
\vspace{-2em}
\else
\[\varepsilon_\text{aDP}(\delta) = \inf_{\alpha>1} \frac12 \varepsilon^2 \alpha + \frac{\log(1/\alpha\delta)}{\alpha-1} + \log(1-1/\alpha) \le  \varepsilon \cdot \left( \sqrt{2\log(1/\delta)} + \varepsilon/2 \right).\]
\fi
\end{lem}
\ificml
We adopt \emph{user-level privacy} in this work -- i.e., each entry in the input corresponds to \emph{all} the records associated with a single user~\citep{mcmahan2018learning}, and thus the privacy guarantee holds with respect to all data belonging to that user. This is stronger than the commonly-used notion of item-level privacy where, if a user contributes multiple records, only the addition or removal of one record is protected. 
We define DP with respect to addition or removing the records of an individual, rather than replacement. Since replacement can be achieved by a combination of an addition and a removal, group privacy (a.k.a.~the triangle inequality) implies a differential privacy guarantee for replacement; however, the privacy parameter will be doubled.
\fi

\else

\paragraph{Organization} 
The remainder of the paper is organized as follows. We summarize our main results and review related works in this section. We present the preliminaries in Section~\ref{sec:prelim}. In Section~\ref{sec:dg}, we introduce the distributed discrete Gaussian mechanism and analyze its privacy guarantees. In Section~\ref{sec:utility}, we show how $\mathbb{R}$-valued vectors can be efficiently mapped to $\mathbb{Z}$-valued vectors and how the distributed discrete Gaussian mechanism can be combined with modulo clipping to obtain noisy vectors in $\mathbb{Z}^d_m$. We present our experimental results in Section \ref{sec:experiments} and conclude our paper with a few interesting and non-trivial extensions in Section~\ref{sec:conclusion}.

\fi

\ificml
\section{Main Results} \label{sec:main-results}
\else
\subsection{Main Results} \label{sec:main-results}
\fi
We start by considering a single round of federated learning in which we are simply summing model update vectors.
That is, we have $n$ clients and assume that each client holds a vector $x_i \in \mathbb{R}^d$ and our goal is to privately approximate $\bar{x} := \sum_i^n x_i$. Client $i$ computes $z_i=\mathcal{A}_{\mathrm{client}}(x_i) \in \mathbb{Z}_m^d$; here, $\mathcal{A}_{\mathrm{client}}(\cdot)$ can be thought of as a compression and privatization scheme. Using secure aggregation as a black box,\footnote{We will assume the secure aggregation protocol accepts $z_i$'s on $\mathbb{Z}_m^d$ (i.e., length-$d$ integer vectors modulo $m$) and computes the sum modulo $m$. Our methods do not depend on the specifics of the implementation of SecAgg.} the server observes 
\begin{equation}\bar{z} := \sum_i^n z_i \! \mod m ~ = \sum_i^n \mathcal{A}_{\mathrm{client}}(x_i) \! \mod m,\end{equation} 
and uses $\bar{z}$ to estimate $\mathcal{A}_{\mathrm{server}}(\bar{z}) \approx \bar{x} = \sum_i^n x_i$.%

The protocol consists of three parts -- the client side $\mathcal{A}_{\mathrm{client}}$, secure aggregation, and the server side $\mathcal{A}_{\mathrm{server}}$. There is already ample work on implementing secure aggregation \cite{bell2020secagg, bonawitz2016practical}; thus we treat SecAgg as a black box which is guaranteed to faithfully compute the modular sum of the inputs, while revealing no further information to a potential privacy adversary. Further discussion of SecAgg and the required trust assumptions is beyond the scope of this work. 
This allows us to focus on the requirements for $\mathcal{A}_{\mathrm{client}}$ and $\mathcal{A}_{\mathrm{server}}$:
\begin{itemize}
    \item \textbf{Privacy:} The sum $\bar{z} = \sum_i^n \mathcal{A}_{\mathrm{client}}(x_i) \! \mod m$ must be a differentially private function of the inputs $x_1, \cdots, x_n$. Specifically, adding or removing one client should only change the distribution of the sum slightly. Note that our requirement is weaker than local DP, since we only reveal the sum, rather than the individual responses $z_i = \mathcal{A}_{\mathrm{client}}(x_i)$.
    
    Privacy is achieved by each client independently adding discrete Gaussian noise \cite{CanonneKS20} to its (appropriately discretized) vector. The sum of independent discrete Gaussians is \emph{not} a discrete Gaussian, but we show that it is \emph{extremely} close for the parameter regime of interest. This is the basis of our differential privacy guarantee, and we believe this result to be of independent interest. 
    
    \item \textbf{Accuracy:} Our goal is to approximate the sum $\mathcal{A}_{\mathrm{server}}(\bar{z}) \approx \bar{x} = \sum_i^n x_i$. For simplicity, we focus on the mean squared error, although our experiments also evaluate the accuracy by aggregating client model updates for federated learning. 
    
    There are three sources of error to consider: (i) the discretization of the $x_i$ vectors from $\mathbb{R}^d$ to $\mathbb{Z}^d$; (ii) the noise added for privacy (which also depends on the norm $\|x_i\|$ and how discretization affects this); and (iii) the potential modular wrap-around introduced by SecAgg modular sum. We provide a detailed analysis of all three effects and how they affect one another.
    
    \item \textbf{Communication and Computation:} It is crucial that our algorithms are efficient, especially the client side, which may be running on a mobile device. Computationally, our algorithms run in time that is nearly linear in the dimension. The communication cost is $O(d \log m)$. While we cannot control the dimension $d$, we can minimize the number of bits per coordinate, which is $\log m$. However, this introduces a tradeoff between communication and accuracy -- larger $m$ means more communication, but we can reduce the probability of a modular wrap around and pick a finer discretization to reduce the rounding error.
\end{itemize}
We focus our discussion on the simple task of summing vectors. In a realistic federated learning system, there will be many summing rounds as we iteratively update our model. Each round will be one invocation of our protocol. The privacy loss parameters of the larger system can be controlled using the composition and subsampling properties of differential privacy. That is, we can use standard privacy accounting techniques \cite{dwork2016concentrated,BunS16,mironov2017renyi,WangBK19} to analyse the more complex system, as long as we have differential privacy guarantees for the basic protocol that is used as a subroutine.

\setlength{\textfloatsep}{10pt}%
\begin{algorithm}[t]
   \caption{Client Procedure $\mathcal{A}_{\mathrm{client}}$}
   \label{alg:client}
\begin{algorithmic}
    \STATE \textbf{Input:} Private vector $x_i \in \mathbb{R}^d$. \hfill \COMMENT{Assume dimension $d$ is a power of $2$.}
    \STATE \textbf{Parameters:} Dimension $d \in \mathbb{N}$; clipping threshold $c>0$; granularity $\gamma>0$; modulus $m \in \mathbb{N}$; noise scale $\sigma>0$; bias $\beta\in[0,1)$.
    \STATE \textbf{Shared/public randomness:} Uniformly random sign vector $\xi \in \{-1,+1\}^d$.
    \STATE Clip and scale vector: $x_i' = \frac1\gamma \min\left\{1, \frac{c}{\|x_i\|_2}\right\} \cdot x_i \in \mathbb{R}^d$.
    \STATE Flatten vector: $x_i'' = H_d D_\xi x_i' \in \mathbb{R}^d$ where $H \in \{-1/\sqrt{d},+1/\sqrt{d}\}^{d \times d}$ is a Walsh-Hadamard matrix satisfying $H^TH=I$ and $D_\xi \in \{-1,0,+1\}^{d \times d}$ is a diagonal matrix with $\xi$ on the diagonal.
    \REPEAT
        \STATE Let $\tilde{x}_i \in \mathbb{Z}^d$ be a randomized rounding of $x_i'' \in \mathbb{R}^d$. I.e., $\tilde{x}_i$ is a product distribution with $\ex{}{\tilde{x}_i}=x_i''$ and $\|\tilde{x}_i - x_i''\|_\infty < 1$.
    \ificml
        \UNTIL{$\|\tilde{x}_i\|_2^2 \! \le \! \min\left\{ \!\!\!\! \begin{array}{c} (c/\gamma+\sqrt{d})^2,\\{\frac{c^2}{\gamma^2} \! + \! \frac{d}{4} \! + \! \sqrt{2\log\left(\frac 1 \beta \right)} \! \cdot \! \left(\frac{c}{\gamma} \! + \! \frac{\sqrt{d}}{2}\right)}\end{array} \!\!\!\! \right\}$.}
    \else
        \UNTIL{$\|\tilde{x}_i\|_2 \le \min\left\{c/\gamma+\sqrt{d}, \sqrt{c^2/\gamma^2 + \frac14 d + \sqrt{2\log(1/\beta)} \cdot \left(c/\gamma + \frac12 \sqrt{d}\right)}\right\}$.}
    \fi
    \STATE Let $y_i \in \mathbb{Z}^d$ consist of $d$ independent samples from the discrete Gaussian $\mathcal{N}_{\mathbb{Z}}(0,\sigma^2/\gamma^2)$.
    \STATE Let $z_i = (\tilde{x}_i + y_i) \mod{m}$.
    \STATE \textbf{Output:} $z_i \in \mathbb{Z}_m^d$ for the secure aggregation protocol.
\end{algorithmic}
\end{algorithm}
\begin{algorithm}[t]
   \caption{Server Procedure $\mathcal{A}_{\mathrm{server}}$}
   \label{alg:server}
\begin{algorithmic}
   \STATE \textbf{Input:} Vector $\bar{z}=(\sum_i^n z_i \mod{m}) \in \mathbb{Z}_m^d$ via secure aggregation.
    \STATE \textbf{Parameters:} Dimension $d \in \mathbb{N}$; number of clients $n \in \mathbb{N}$; clipping threshold $c>0$; granularity $\gamma>0$; modulus $m \in \mathbb{N}$; noise scale $\sigma>0$; bias $\beta\in[0,1)$.
    \STATE \textbf{Shared/public randomness:} Uniformly random sign vector $\xi \in \{-1,+1\}^d$.
    \STATE Map $\mathbb{Z}_m$ to $\{1-m/2,2-m/2,\cdots, -1, 0, 1 \cdots, m/2-1, m/2\}$ so that $\bar{z}$ is mapped to $\bar{z}' \in [-m/2,m/2]^d \cap \mathbb{Z}^d$ (and we have $\bar{z}' \mod{m} = \bar{z}$).
    \STATE \textbf{Output:} $y \!=\! \gamma D_\xi H_d^T \bar{z}' \in \mathbb{R}^d$. \ificml \else \hfill \fi \COMMENT{Goal: $y \approx \bar{x} = \sum_i^n x_i$}
\end{algorithmic}
\end{algorithm}

We now present our algorithm in two parts -- the client part $\mathcal{A}_{\mathrm{client}}$ in Algorithm \ref{alg:client} and the server part $\mathcal{A}_{\mathrm{server}}$ in Algorithm \ref{alg:server}.
The two parts are connected by a secure aggregation protocol. We also note that our algorithms may be a subroutine of a larger FL system. 

We briefly remark about the parameters of the algorithm: 
$d$ is the dimension of the inputs $x_i$ and outputs, which we assume is a power of $2$ for convenience. %
The input vectors must have their norm clipped for privacy; $c$ controls this tradeoff -- larger $c$ will require more noise for privacy (larger $\sigma$) and smaller $c$ will distort the vectors more.
If $\beta=0$, then the discretization via randomized rounding is unbiased, but the norm of $\tilde{x}_i$ could be larger; each iteration of the randomized rounding loop succeeds with probability at least $1-\beta$. The modulus $m$ will determine the communication complexity -- $z_i$ requires $d \log_2 m$ bits to represent. The noise scale $\sigma$ determines the privacy, specifically $\varepsilon \approx c / \sqrt{n} \sigma$. Finally, the granularity $\gamma$ gives a tradeoff: smaller $\gamma$ means the randomized rounding introduces less error, but also makes it more likely that the modulo $m$ operation introduces error.

We also remark about some of the techniques used in our system: 
The first step in Algorithm \ref{alg:client} scales and clips the input vector so that $\|x_i'\|_2 \le c/\gamma$. 
The next step performs a unitary rotation/reflection operation $x_i'' = H_d D_\xi x_i'$ \cite{suresh2017distributed}. This operation ``flattens'' the vector -- i.e., $\|x_i''\|_\infty \approx \frac{1}{\sqrt{d}} \|x_i'\|_2$. Flattening ensures that the modular arithmetic does not introduce large distortions due to modular wrap around (i.e., large coordinates of $x_i''$ will be subject to modular reduction).
This flattening operation and the scaling by $\gamma$ are undone in the last step of Algorithm \ref{alg:server}.
The $x_i''$ is randomly rounded to the integer grid in an unbiased manner. That is, each coordinate is independently rounded to one of the two nearest integers. E.g., $42.3$ has a $30\%$ probability of being rounded up to $43$ and a $70\%$ probability of being rounded down to $42$. This may increase the norm -- $\| \tilde x_i \|_2 \le \|x_i''\|_2 + \sqrt{d}$.
To mitigate this, we perform \emph{conditional} randomized rounding: repeatedly perform independent randomized rounding on $x_i''$ until $\|\tilde x_i\|_2$ is not too big. This introduces a small amount of bias, but, since the noise we add to attain differential privacy must scale with the norm of the discretized vector, reducing the norm reduces the noise variance.

\ificml
\textbf{Privacy}\quad
\else
\paragraph{Privacy}
\fi
We now state the privacy of our algorithm.

\begin{thm}[Privacy of Our Algorithm]\label{thm:intro-privacy}
Let $c, d, \gamma, \beta, \sigma$ be the parameters of Algorithm \ref{alg:client} and $n$ the number of trustworthy clients. Define
\begin{align}
    \Delta_2^2 &:= \min \left\{ \!\! \begin{array}{c} c^2 \!+\! \frac{\gamma^2 d}{4} \!+\! \sqrt{2\log\left(\frac 1 \beta \right)} \!\cdot\! \gamma \!\cdot\! \left(c \!+\! \frac{\gamma}{2} \sqrt{d}\right) ,\\ \left( c + \gamma \sqrt{d} \right)^2 \end{array} \!\! \right\}, \\
    \tau &:= 10 \cdot \sum_{k=1}^{n-1} e^{-2\pi^2\frac{\sigma^2}{\gamma^2} \cdot \frac{k}{k+1}}, \\
    \varepsilon &:= \min \left\{ \begin{array}{c} \sqrt{ \frac{\Delta_2^2}{n\sigma^2} + 2 \tau d} ,\\ \frac{\Delta_2}{\sqrt{n} \sigma} + \tau \sqrt{d} \end{array} \right\}. 
\end{align} 
Then Algorithm \ref{alg:client} satisfies $\frac12 \varepsilon^2$-concentrated differential privacy,\footnote{\ificml See the supplementary material or \citet{BunS16} for a formal definition. \fi Note that this is with respect to the addition or removal of an individual, not replacement (which would double the $\varepsilon$ parameter). To keep $n$ fixed, we could define addition/removal to simply zero-out the relevant vectors.} assuming that secure aggregation only reveals the sum $z=(\sum_i^n z_i \!\mod{m}) \in \mathbb{Z}_m^d$ to the privacy adversary.
\end{thm}

We remark on the parameters of the theorem: To first approximation, $\varepsilon \approx \frac{c}{\sqrt{n}\sigma}$. This is because the input vectors are clipped to have norm $c$ and then each client adds (discrete) Gaussian noise with variance $\approx \sigma^2$. The noise added to the sum thus has variance $\approx n \sigma^2$. However, there are two additional effects to account for: First, randomized rounding can increase the norm from $c$ to $\Delta_2$ and this becomes the sensitivity bound that we use for the privacy analysis. Second, the sum of $n$ discrete Gaussians is \emph{not} a discrete Gaussian, but it is close; $\tau$ bounds the max divergence between the sum of $n$ discrete Gaussians each with scale parameter $\sigma/\gamma$ and one discrete Gaussian with scale parameter $\sqrt{n} \sigma / \gamma$.

\ifarxiv
Note that $\frac12\varepsilon^2$-concentrated DP \cite{dwork2016concentrated, BunS16} is equivalent to satisfying $\left(\alpha, \frac12\varepsilon^2\alpha\right)$-R\'enyi DP \cite{mironov2017renyi} simultaneously for all $\alpha>1$. Concentrated DP can be converted to the more standard approximate differential privacy \cite{,CanonneKS20}: For any $\delta>0$, $\frac12\varepsilon^2$-concentrated DP implies $\left( \varepsilon_{\text{aDP}}(\delta), \delta\right)$-DP, where 
\ificml
\begin{align*}
    \varepsilon_{\text{aDP}}(\delta) &= \inf_{\alpha>1} \frac12{\varepsilon}^2\alpha + \frac{\log(1/\alpha\delta)}{\alpha-1} + \log(1-1/\alpha) \\
    &\le \frac12{\varepsilon}^2 + \sqrt{2\log(1/\delta)}\cdot{\varepsilon}.
\end{align*}
\else
\[\varepsilon_{\text{aDP}}(\delta) = \inf_{\alpha>1} \frac12{\varepsilon}^2\alpha + \frac{\log(1/\alpha\delta)}{\alpha-1} + \log(1-1/\alpha) \le \frac12{\varepsilon}^2 + \sqrt{2\log(1/\delta)}\cdot{\varepsilon}.\]
\fi
\fi

\ificml
\textbf{Accuracy}\quad
\else
\paragraph{Accuracy}
\fi
Next we turn to the accuracy of the algorithm. We provide both an empirical evaluation and theoretical analysis. We give the following asymptotic guarantee; a more precise guarantee with exact constants can be found in \ificml the accompanying supplementary material\else Theorem \ref{thm:main-guarantee}\fi.

\begin{thm}[Accuracy of Our Algorithm]\label{thm:intro-accuracy}
Let $n, m, d \in \mathbb{N}$ and $c,\varepsilon>0$ satisfy \[m \ge \tilde{O}\left( n  + \sqrt{\frac{\varepsilon^2 n^3}{d}} + \frac{\sqrt{d}}{\varepsilon}\right).\]
Let $\tilde{A}(x) = \mathcal{A}_{\mathrm{server}}\left( \sum_i^n \mathcal{A}_{\mathrm{client}}(x_i) \mod m \right)$ denote the output of the system given by Algorithms \ref{alg:client} and \ref{alg:server} instantiated with parameters $\gamma = \tilde{\Theta}\left(\frac{c n}{m \sqrt{d}} + \frac{c}{\varepsilon m}\right)$, $\beta \le \Theta\left( \frac 1 n \right)$, and $\sigma = \tilde{\Theta}\left(\frac{c}{\varepsilon\sqrt{n}}+\sqrt{\frac{d}{n}}\cdot\frac{\gamma}{\varepsilon}\right)$.
Then $\tilde{A}$ satisfies $\frac12\varepsilon^2$-concentrated differential privacy and attains the following accuracy.
Let $x_1, \cdots, x_n \in \mathbb{R}^d$ with $ \|x_i\|_2 \le c$ for all $i \in [n]$. Then 
\begin{equation}
    \ex{}{\left\|\tilde{A}(x) - \sum_i^n x_i\right\|_2^2} \le {O} \left(\frac{c^2 d}{\varepsilon^2} \right). \label{eqn:intro-accuracy-bound}
\end{equation}
\end{thm}
To interpret Theorem \ref{thm:intro-accuracy}, note that mean squared error $O\left(\frac{c^2 d}{\varepsilon^2} \right)$ is, up to constants, exactly the error we would expect to attain for differential privacy in the central model. Our analysis attains reasonably sharp constants (at the expense of many lower order terms that we suppress here in the introduction). However, to truly gauge the practicality of our method, we perform an empirical evaluation.

\paragraph{Experiments}
To investigate the interplay between communication, accuracy, and privacy under our proposed protocol in practice, we empirically evaluate our protocol and compare it to the commonly used centralized continuous Gaussian mechanism on two canonical tasks: distributed mean estimation (DME) and federated learning (FL).
For DME, each client holds a vector and the server’s goal is to obtain a differentially private mean estimate of the vectors. We show that 16 bits per coordinate are sufficient to nearly match the utility of the Gaussian baseline for regimes of interest.
For FL, we show on Federated EMNIST~\cite{caldas2018leaf} and Stack Overflow~\cite{tffauthors2019} that our approach gives good performance under tight privacy budgets, despite using generic RDP amplification via sampling~\cite{zhu2019poission} for our methods and the precise RDP analysis for the subsampled Gaussian mechanism \cite{mironov2019r}.
We provide an open-source implementation of our methods in TensorFlow Privacy~\cite{andrew2019differentially} and TensorFlow Federated \cite{ingerman2019tff}.\footnote{Code: \url{https://github.com/google-research/federated/tree/master/distributed_dp}.}

\ificml

\else

\fi

\ificml
\section{Related Work}
\label{sec:related_work}
\else
\subsection{Related Work}
\label{sec:related_work}
\fi

\ificml
\textbf{Federated Learning}\quad
\else
\paragraph{Federated Learning}  
\fi
Under FL, a set of clients (e.g., mobile devices or institutions) collaboratively train a model under the orchestration of a central server, while keeping training data decentralized \cite{mcmahan2017communication,bonawitz19sysml}. It embodies the principles of focused data collection and minimization, and can mitigate many of the systemic privacy risks and costs resulting from traditional, centralized machine learning and data science approaches. 
FL performs many rounds of interaction between the server and subsets of online clients; for example, each round may consist of computing and aggregating the gradients of the loss for a given set of model weights, which are then updated using the aggregated gradients for the next round. This allows us to focus on the simple task of computing the sum of vectors (model updates) held by the clients.
We refer the reader to~\ificml\citet{kairouz2019advances}\else\cite{kairouz2019advances}\fi~for a survey of recent advances and open problems in FL. 

While the above features can offer significant practical privacy improvements over centralizing training data, FL offers no formal guarantee of privacy and has to be composed with other privacy technologies to offer strong (worst-case) privacy guarantees. The primary goal of this paper is to show how two such technologies, namely secure aggregation and differential privacy, can be carefully combined with FL to offer strong and quantifiable privacy guarantees.  

\ificml
\textbf{Secure Aggregation}\quad
\else
\paragraph{Secure Aggregation}
\fi
SecAgg is a lightweight instance of cryptographic secure multi-party computation (MPC) that enables clients to submit vector inputs, such that the server learns just an aggregate function of the clients’ vectors, typically the sum. In most contexts of FL, single-server SecAgg is achieved via additive masking over a finite group \cite{bell2020secagg, bonawitz2016practical}. To be precise, clients add randomly sampled zero-sum mask vectors by working in the space of integers modulo $m$ and sampling the coordinates of the mask uniformly from $\mathbb{Z}_m$. This process guarantees that each client's masked update is indistinguishable from random values. However, when all the masked updates are summed modulo $m$ by the server, the masks cancel out and the server obtains the exact sum. Observe that in practice, the model updates computed by the clients are real valued vectors whereas SecAgg requires the input vectors to be from $\mathbb{Z}_m$ (i.e., integers modulo $m$). This discrepancy is typically bridged by clipping the values to a fixed range, say $[-r, r]$, which is then translated and scaled to $\left[0,\frac{m-1}{n} \right]$, and then uniformly quantizing the values in this range to integers in $\{0, 1, \cdots, \lfloor \frac{m-1}{n}\rfloor\}$, where $n$ is the number of clients. This ensures that, up to clipping and quantization, the server computes the exact sum without overflowing (i.e., the sum is in $[0,m-1]$, which is unaffected by the modular arithmetic) \cite{bonawitz2019federated}. In our work, we provide a novel strategy for transforming $\mathbb{R}$-valued vectors into $\mathbb{Z}_m$-valued ones. 

\ificml
\textbf{Distributed DP}\quad
\else
\paragraph{Distributed DP}
\fi
While SecAgg prevents the server from inspecting individual client updates, the server is still able to learn the sum of the updates, which itself may leak potentially sensitive information \cite{melis2019exploiting, carlini2019secret, song2019auditing, dwork2015robust, song2019overlearning, nasr2021adversary, shokri2017membership}. To address this issue, differential privacy (DP) \cite{DworkMNS06}, and in particular, DP-SGD can be employed~\cite{song2013stochastic, bassily2014private, abadi2016deep, tramer2020differentially}. DP is a rigorous measure of information disclosure about individuals participating in computations over centralized or distributed datasets. Over the last decade, an extensive set of techniques has been developed for differentially private data analysis, particularly under the assumption of a centralized setting, where the raw data is collected by a trusted service provider prior to applying perturbations necessary to achieve privacy. This setting is commonly referred to as the central DP setting. More recently, there has been a great interest in the local model of DP \cite{kasiviswanathan2008ldp, evfimievski2004privacy, warner1965randomized} where the data is perturbed on the client side before it is collected by a service provider. 

Local DP avoids the need for a fully trusted aggregator. However, it is now well-established that local DP usually leads to a steep hit in accuracy \cite{kasiviswanathan2008ldp, duchi2013local, kairouz2016discrete}. In order to recover some of the utility of central DP, without having to rely on a fully trusted central server, an emerging set of models of DP, often referred to as distributed DP, can be used. Under distributed DP, clients employ a cryptographic protocol (e.g., SecAgg) to simulate some of the benefits of a trusted central party. Clients first compute minimal application-specific reports, perturb these slightly, and then execute the aggregation protocol. The untrusted server then only has access to the aggregated reports, with the aggregated perturbations.   %
The noise added by individual clients is typically insufficient for a meaningful local DP guarantee on its own. However, after aggregation, the aggregated noise is sufficient for a meaningful DP guarantee, under the security assumptions necessary for the cryptographic protocol. 

\ificml
\textbf{FL with SecAgg and Distributed DP}\quad
\else
\paragraph{FL with SecAgg and Distributed DP}
\fi
Despite the recent surge of interest in distributed DP, much of the work in this space focuses on the shuffled model of DP where a trusted third party (or a trusted execution environment) shuffles the noisy client updates before forwarding them to the server  \cite{erlingsson2019amplification, bittau17prochlo, cheu2019distributed}. For more information on the shuffled model of DP, we refer the reader to \citet{GKMP20-icml, anon-power, ghazi2019private, ghazi2020pure, ishai2006cryptography, BalleBGN19, balle_merged, balcer2019separating, balcer2021connecting, girgis2020shuffled}. 

The combination of SecAgg and distributed DP in the context of communication-efficient FL is far less studied. For instance, the majority of existing works ignore the finite precision and modular summation arithmetic associated with secure aggregation \cite{goryczka2013secure, truex2019hybrid, valovich2017computational}. This is especially problematic at low SecAgg bit-widths (e.g., in practical FL settings where communication efficiency is critical). 

The closest work to ours is \textsf{cpSGD} \cite{agarwal2018cpsgd}, which also serves as an inspiration for much of our work. \textsf{cpSGD} uses a distributed version of the binomial mechanism \cite{DworkKMMN06} to achieve distributed DP. When properly scaled, the binomial mechanism can (asymptotically) match the continuous Gaussian mechanism. However, there are several important differences between our work and \textsf{cpSGD}. First, the binomial mechanism does not achieve R\'enyi or concentrated DP \cite{dwork2016concentrated,BunS16,mironov2017renyi} and hence we cannot combine it with state-of-the-art composition and subsampling results, which is a significant barrier if we wish to build a larger FL system. The binomial mechanism is analyzed via approximate DP; in other words, the privacy loss for the binomial mechanism can be infinite with a non-zero probability. We avoid this issue by basing our privacy guarantee on the discrete Gaussian mechanism \cite{CanonneKS20}, which also matches the performance of the continuous Gaussian and yields clean concentrated DP guarantees that are suitable for sharp composition and subsampling analysis. \textsf{cpSGD} also does not consider the impact of modular arithmetic, which makes it harder to combine with secure aggregation. 

Previous attempts at achieving DP using a distributed version of the discrete Gaussian mechanism have either inaccurately glossed over the fact that the sum of discrete Gaussians is not a discrete Gaussian, or assumed that all clients secretly share a seed that is used to generate the same discrete Gaussian instance, which is problematic because a single honest-but-curious client can fully break the privacy guarantees \cite{lun2021}. We provide a careful privacy analysis for sums of discrete Gaussians. Our privacy guarantees degrade gracefully as a function of the fraction of malicious (or dropped out) clients.

\section{Preliminaries}\label{sec:prelim}

\section{Distributed Discrete Gaussian}\label{sec:dg}

We will use the discrete Gaussian \cite{CanonneKS20} as the basis of our privacy guarantee.
\begin{defn}[Discrete Gaussian]
The discrete Gaussian with scale parameter $\sigma > 0$ and location parameter $\mu \in \mathbb{Z}$ is a probability distribution supported on the integers $\mathbb{Z}$ denoted by $\mathcal{N}_{\mathbb{Z}}(\mu,\sigma^2)$ and defined by \begin{equation}\forall x \in \mathbb{Z} ~~~~ \pr{X \gets \mathcal{N}_{\mathbb{Z}}(\mu,\sigma^2)}{X=x} = \frac{\exp\left(\frac{-(x-\mu)^2}{2\sigma^2}\right)}{\sum_{y \in \mathbb{Z}} \exp\left(\frac{-(y-\mu)^2}{2\sigma^2}\right)}.\nonumber
\end{equation}
\end{defn}

The discrete Gaussian has many of the desirable properties of the continuous Gaussian \cite{CanonneKS20}, including the fact that it can be used to provide differential privacy.
\begin{thm}[Privacy of the Discrete Gaussian]\label{thm:dgauss-priv}
Let $\sigma>0$ and $\mu,\mu'\in\mathbb{Z}$. Then%
\ificml
, for all $\alpha>1$,
\begin{equation}\dr{\alpha}{\mathcal{N}_{\mathbb{Z}}(\mu,\sigma^2)}{\mathcal{N}_{\mathbb{Z}}(\mu',\sigma^2)} \le \alpha \cdot \frac{(\mu-\mu')^2}{2\sigma^2},\end{equation}
where $\dr{\alpha}{P}{Q}$ is the R\'enyi divergence of order $\alpha$.
\else
\begin{equation}\dr{*}{\mathcal{N}_{\mathbb{Z}}(\mu,\sigma^2)}{\mathcal{N}_{\mathbb{Z}}(\mu',\sigma^2)} = \frac{(\mu-\mu')^2}{2\sigma^2}.\end{equation}
\fi
\end{thm}

Unlike the continuous Gaussian, the sum/convolution of two independent discrete Gaussians is \emph{not} a discrete Gaussian. However, we show that, for reasonable parameter settings, it is very close to one. The following result is a simpler version of Theorem 4.6 of\citet{genise2020improved}.

\begin{thm}[Convolution of two Discrete Gaussians]\label{thm:Convolution}
Let $\sigma,\tau\ge\frac12$.
Let $X \gets \mathcal{N}_{\mathbb{Z}}(0,\sigma^2)$ and $Y \gets \mathcal{N}_{\mathbb{Z}}(0,\tau^2)$ be independent. Let $Z=X+Y$. Let $W \gets \mathcal{N}_{\mathbb{Z}}(0,\sigma^2+\tau^2)$. Then
\ificml
\begin{align*}
\dr{\pm\infty}{Z}{W}&=&\sup_{z \in \mathbb{Z}} \left|\log\left(\frac{\mathbb{P}[Z=z]}{\mathbb{P}[W=z]}\right)\right| && \\
&\le& 5 \cdot e^{-2\pi^2/(1/\sigma^2+1/\tau^2)}.&&\\
\end{align*}
\else
\begin{equation}\dr{\pm\infty}{Z}{W}=\sup_{z \in \mathbb{Z}} \left|\log\left(\frac{\mathbb{P}[Z=z]}{\mathbb{P}[W=z]}\right)\right|\le 5 \cdot e^{-2\pi^2/(1/\sigma^2+1/\tau^2)}.\end{equation}
\fi
\end{thm}
\ificml
\vspace{-2em}
\fi
The bound of the theorem is surprisingly strong; if $\sigma^2=\tau^2=3$, then the bound is $\le 10^{-12}$, which should suffice for most applications. Furthermore, closeness in max divergence is the strongest measure of closeness that we could hope for (rather than, say, total variation distance).

\ifarxiv
\begin{proof}
For all $z \in \mathbb{Z}$,
\begin{align*}
\mathbb{P}[Z=z] &= \sum_{x \in \mathbb{Z}} \mathbb{P}[X=x] \cdot \mathbb{P}[Y=z-x]\\
&= \sum_{x \in \mathbb{Z}} \frac{e^{-x^2/2\sigma^2}}{\sum_{u \in \mathbb{Z}} e^{-u^2/2\sigma^2}} \frac{e^{-(x-z)^2/2\tau^2}}{\sum_{v \in \mathbb{Z}} e^{-v^2/2\tau^2}} \\
&= \frac{\sum_{x \in \mathbb{Z}} \exp\left(- \frac{(\tau^2+\sigma^2) x^2 - 2\sigma^2xz + \sigma^2z^2}{2\sigma^2\tau^2}\right)}{\left(\sum_{u \in \mathbb{Z}} e^{-u^2/2\sigma^2}\right)\cdot\left(\sum_{v \in \mathbb{Z}} e^{-v^2/2\tau^2}\right)} \\
&= \frac{\sum_{x \in \mathbb{Z}} \exp\left(- \frac{x^2 - 2\frac{\sigma^2}{\tau^2+\sigma^2}xz + \frac{\sigma^2}{\tau^2+\sigma^2}z^2}{2\frac{\sigma^2\tau^2}{\tau^2+\sigma^2}}\right)}{\left(\sum_{u \in \mathbb{Z}} e^{-u^2/2\sigma^2}\right)\cdot\left(\sum_{v \in \mathbb{Z}} e^{-v^2/2\tau^2}\right)} \\
&= \frac{\sum_{x \in \mathbb{Z}} \exp\left(- \frac{\left(x - \frac{\sigma^2}{\tau^2+\sigma^2}z\right)^2 - \left(\frac{\sigma^2}{\tau^2+\sigma^2}z\right)^2+ \frac{\sigma^2}{\tau^2+\sigma^2}z^2}{2\frac{\sigma^2\tau^2}{\tau^2+\sigma^2}}\right)}{\left(\sum_{u \in \mathbb{Z}} e^{-u^2/2\sigma^2}\right)\cdot\left(\sum_{v \in \mathbb{Z}} e^{-v^2/2\tau^2}\right)} \\
&= \exp\left(\frac{\left(\frac{\sigma^2}{\tau^2+\sigma^2}z\right)^2 - \frac{\sigma^2}{\tau^2+\sigma^2}z^2}{2\frac{\sigma^2\tau^2}{\tau^2+\sigma^2}}\right)\frac{\sum_{x \in \mathbb{Z}} \exp\left(- \frac{\left(x - \frac{\sigma^2}{\tau^2+\sigma^2}z\right)^2}{2\frac{\sigma^2\tau^2}{\tau^2+\sigma^2}}\right)}{\left(\sum_{u \in \mathbb{Z}} e^{-u^2/2\sigma^2}\right)\cdot\left(\sum_{v \in \mathbb{Z}} e^{-v^2/2\tau^2}\right)} \\
&= \exp\left(\frac{-z^2}{2(\tau^2+\sigma^2)}\right)\frac{\sum_{x \in \mathbb{Z}} \exp\left(- \frac{\left(x - \frac{\sigma^2}{\tau^2+\sigma^2}z\right)^2}{2\frac{\sigma^2\tau^2}{\tau^2+\sigma^2}}\right)}{\left(\sum_{u \in \mathbb{Z}} e^{-u^2/2\sigma^2}\right)\cdot\left(\sum_{v \in \mathbb{Z}} e^{-v^2/2\tau^2}\right)} .
\end{align*}
The $\exp\left(\frac{-z^2}{2(\tau^2+\sigma^2)}\right)$ term is exactly what we want -- it is, up to scaling, $\mathbb{P}[W=z]$. The denominator is a constant (i.e., it does not depend on $z$), which means we do not need to worry about it. The troublesome term is \begin{align*}
    \sum_{x \in \mathbb{Z}} \exp\left(- \frac{\left(x - \frac{\sigma^2}{\tau^2+\sigma^2}z\right)^2}{2\frac{\sigma^2\tau^2}{\tau^2+\sigma^2}}\right) &= \sum_{x \in \mathbb{Z}} \exp\left(- \frac{\left(x - \frac{\sigma^2}{\tau^2+\sigma^2}z\right)^2}{2}\left(\frac{1}{\sigma^2} + \frac{1}{\tau^2}\right)\right)\\
    &= \sum_{x \in \mathbb{Z}} f_{(1/\sigma^2+1/\tau^2)^{-1}}\left(x - \frac{\sigma^2}{\tau^2+\sigma^2}z\right),
\end{align*} where $f_{\rho^2}(x) := \exp\left(\frac{-x^2}{2\rho^2}\right)$.
Now we apply the Poisson summation formula using the Fourier transform $\hat{f}_{\rho^2}(y) = \sqrt{2\pi\rho^2} \cdot \exp(-2\pi^2\rho^2y^2)$.
For $t,\rho\in\mathbb{R}$, we have
\begin{align*}
g_{\rho^2}(t) := \sum_{x \in \mathbb{Z}} f_{\rho^2}(x-t) &= \sum_{y \in \mathbb{Z}} \hat{f}_{\rho^2}(y) \cdot e^{-2\pi\sqrt{-1}yt} \\
&= \sqrt{2\pi\rho^2}  \sum_{y \in \mathbb{Z}} e^{-2\pi^2\rho^2y^2} \cdot e^{-2\pi\sqrt{-1}yt} \\
&= \sqrt{2\pi\rho^2}  \sum_{y \in \mathbb{Z}} e^{-2\pi^2\rho^2y^2} \cdot \cos(2\pi y t) \\
&= \sqrt{2\pi\rho^2}  \left( 1 + 2 \sum_{n=1}^\infty e^{-2\pi^2\rho^2n^2} \cdot \cos(2\pi n t) \right).
\end{align*}
Note that $g_{\rho^2}(t) \le g_{\rho^2}(0)$ for all $t,\rho\in\mathbb{R}$ \cite[Lemma 6]{CanonneKS20}, since $\cos(2\pi n t) \le 1 = \cos (2 \pi n 0)$ for all $n$ and $t$. Our goal is to prove a lower bound on $g_{\rho^2}(t) / g_{\rho^2}(0)$, which follows from the following bound:
\begin{align*}
g_{\rho^2}(0) - g_{\rho^2}(t)
&= \sqrt{2\pi\rho^2}  \cdot  2 \sum_{n=1}^\infty e^{-2\pi^2\rho^2n^2} \cdot \left( 1 - \cos(2\pi n t) \right)\\
&\le \sqrt{2\pi\rho^2}  \cdot  2 \sum_{n=1}^\infty e^{-2\pi^2\rho^2n^2} \cdot 2\\
&= 4 \sqrt{2\pi\rho^2} \cdot e^{-2\pi^2\rho^2} \cdot \sum_{n=1}^\infty e^{-2\pi^2\rho^2(n^2-1)} \\
&\le 4 \sqrt{2\pi\rho^2} \cdot e^{-2\pi^2\rho^2} \cdot \sum_{n=1}^\infty e^{-6\pi^2\rho^2(n-1)} \tag{$n^2-1=(n+1)(n-1) \ge 3(n-1)$.}\\
&= 4 \sqrt{2\pi\rho^2} \cdot e^{-2\pi^2\rho^2} \cdot \frac{1}{1- e^{-6\pi^2\rho^2}} \\
&\le 4\frac{e^{-2\pi^2\rho^2}}{1- e^{-6\pi^2\rho^2}} \cdot g_{\rho^2}(0). \tag{$g_{\rho^2}(0) \ge \sqrt{2\pi\rho^2}$.}
\end{align*}
Thus we obtain the bound \[ 1 - 4\frac{e^{-2\pi^2\rho^2}}{1- e^{-6\pi^2\rho^2}} \le \frac{g_{\rho^2}(t)}{g_{\rho^2}(0)} \le 1.\]

For any $z \in \mathbb{Z}$, we have 
\begin{align*}
\frac{\mathbb{P}[Z=z]}{\mathbb{P}[W=z]} &= \frac{g_{(1/\sigma^2+1/\tau^2)^{-1}}(0) \cdot \sum_{w\in\mathbb{Z}} e^{-w^2/2(\sigma^2+\tau^2)}}{\left(\sum_{u \in \mathbb{Z}} e^{-u^2/2\sigma^2}\right)\cdot\left(\sum_{v \in \mathbb{Z}} e^{-v^2/2\tau^2}\right)} \cdot \frac{g_{(1/\sigma^2+1/\tau^2)^{-1}}\left(\frac{\sigma^2}{\sigma^2+\tau^2} z\right)}{g_{(1/\sigma^2+1/\tau^2)^{-1}}(0)} \\
&= c(\sigma^2,\tau^2) \cdot \frac{g_{(1/\sigma^2+1/\tau^2)^{-1}}\left(\frac{\sigma^2}{\sigma^2+\tau^2} z\right)}{g_{(1/\sigma^2+1/\tau^2)^{-1}}(0)}\\
&\in \left[c(\sigma^2,\tau^2)\cdot\left( 1 - 4\frac{e^{-2\pi^2/(1/\sigma^2+1/\tau^2)}}{1- e^{-6\pi^2/(1/\sigma^2+1/\tau^2)}}\right),c(\sigma^2,\tau^2)\right].
\end{align*}
Note that this interval is independent of $z$. Here $c(\sigma^2,\tau^2)$ is an appropriate constant.

The interval must contain $1$, since $Z$ and $W$ are both probability distributions. Thus  $c(\sigma^2,\tau^2) \ge 1$ and $c(\sigma^2,\tau^2)\cdot\left( 1 - 4\frac{e^{-2\pi^2/(1/\sigma^2+1/\tau^2)}}{1- e^{-6\pi^2/(1/\sigma^2+1/\tau^2)}}\right) \le 1$, whence, for all $z \in \mathbb{Z}$, \[ 1 - 4\frac{e^{-2\pi^2/(1/\sigma^2+1/\tau^2)}}{1- e^{-6\pi^2/(1/\sigma^2+1/\tau^2)}} \le \frac{\mathbb{P}[Z=z]}{\mathbb{P}[W=z]} \le \frac{1}{1 - 4\frac{e^{-2\pi^2/(1/\sigma^2+1/\tau^2)}}{1- e^{-6\pi^2/(1/\sigma^2+1/\tau^2)}}}\]
and \[ \left|\log\left(\frac{\mathbb{P}[Z=z]}{\mathbb{P}[W=z]}\right)\right|\le -\log\left(1 - 4\frac{e^{-2\pi^2/(1/\sigma^2+1/\tau^2)}}{1- e^{-6\pi^2/(1/\sigma^2+1/\tau^2)}}\right) \le 5 \cdot e^{-2\pi^2/(1/\sigma^2+1/\tau^2)},\] as long as $1/\sigma^2+1/\tau^2 \le 8$.
\end{proof}
\fi

\ificml
 Theorem \ref{thm:Convolution} can easily be extended to sums of more than two discrete Gaussians by the triangle inequality and to the multivariate setting by composition.
 Combining with Theorem \ref{thm:dgauss-priv} yields our privacy result:
\else
Theorem \ref{thm:Convolution} can easily be extended to sums of more than two discrete Gaussians by induction: 

\begin{cor}[Convolution of Many Discrete Gaussians]\label{cor:n-conv}
Let $\sigma \ge \frac12$.
Let $X_i \gets \mathcal{N}_{\mathbb{Z}}(0,\sigma^2)$ independently for each $i$. Let $Z_n=\sum_i^n X_i$. Let $W_n \gets \mathcal{N}_{\mathbb{Z}}(0, n \cdot \sigma^2)$. Then
\ificml
\begin{align*}
\dr{\pm\infty}{Z_n}{W_n}&=\sup_{z \in \mathbb{Z}} \left|\log\left(\frac{\mathbb{P}[Z_n=z]}{\mathbb{P}[W_n=z]}\right)\right| &&\\
&\le 5 \cdot \sum_{k=1}^{n-1} e^{-2\pi^2\sigma^2\frac{k}{k+1}}&& \\ 
&\le 5(n-1) e^{-\pi^2\sigma^2}.&&
\end{align*}
\else
\begin{equation}\dr{\pm\infty}{Z_n}{W_n}=\sup_{z \in \mathbb{Z}} \left|\log\left(\frac{\mathbb{P}[Z_n=z]}{\mathbb{P}[W_n=z]}\right)\right|\le 5 \cdot \sum_{k=1}^{n-1} e^{-2\pi^2\sigma^2\frac{k}{k+1}} \le 5(n-1) e^{-\pi^2\sigma^2}.\end{equation}
\fi
\end{cor}

\ifarxiv
\begin{proof}
Let $\tilde{Z}_n = W_{n-1} + X_n$. (Note that $Z_n = Z_{n-1} + X_n$.) By the triangle inequality and postprocessing,
\begin{align*}
    \dr{\pm\infty}{Z_n}{W_n} &\le \dr{\pm\infty}{Z_n}{\tilde{Z}_n} + \dr{\pm\infty}{\tilde{Z}_n}{W_n}\\
    &\le \dr{\pm\infty}{Z_{n-1}}{W_{n-1}} + \dr{\pm\infty}{\tilde{Z}_n}{W_n}.
\end{align*}
By Theorem \ref{thm:Convolution}, \[\dr{\pm\infty}{\tilde{Z}_n}{W_n} \le 5 \cdot e^{-2\pi^2/(1/\sigma^2+1/(n-1)\sigma^2)} = 5 \cdot e^{-2\pi^2 \sigma^2 (n-1)/n}.\]
The result now follows by induction; the base case $n=1$ is trivial.
\end{proof}
\fi

We can now use the triangle inequality to combine our convolution closeness results with the privacy guarantee of a single discrete Gaussian to obtain a privacy guarantee for sums of discrete Gaussians:

\begin{prop}[Privacy for Sums of Discrete Gaussians]\label{prop:dg-sum}
Let $\sigma \ge \frac12$.
Let $X_i \gets \mathcal{N}_{\mathbb{Z}}(0,\sigma^2)$ independently for each $i$. Let $Z_n=\sum_i^n X_i$.
Then, for all $\Delta \in \mathbb{Z}$ and all $\alpha \in [1,\infty)$, \begin{equation}\dr{\alpha}{Z_n}{Z_n+\Delta} \le \min \left\{ \frac{\alpha\Delta^2}{2n\sigma^2} +  10 \cdot \sum_{k=1}^{n-1} e^{-2\pi^2\sigma^2\frac{k}{k+1}}, \frac{\alpha}{2} \cdot \left(\frac{|\Delta|}{\sqrt{n} \sigma} + 10 \cdot \sum_{k=1}^{n-1} e^{-2\pi^2\sigma^2\frac{k}{k+1}}\right)^2 \right\}.\end{equation}
That is, an algorithm $M$ that adds $Z_n$ to a sensitivity-$\Delta$ query satisfies $\frac12\varepsilon^2$-concentrated differential privacy for $\varepsilon = \min \left\{ \sqrt{ \frac{\Delta^2}{n\sigma^2} +  20 \cdot \sum_{k=1}^{n-1} e^{-2\pi^2\sigma^2\frac{k}{k+1}}} ,\frac{|\Delta|}{\sqrt{n} \sigma} + 10 \cdot \sum_{k=1}^{n-1} e^{-2\pi^2\sigma^2\frac{k}{k+1}} \right\}$.
\end{prop}
To make the above bound concrete, if $\sigma=\Delta=1$ and $n=10^4$, then $\varepsilon < 0.02$.

\ifarxiv
\begin{proof}
Let $W \gets \mathcal{N}_{\mathbb{Z}}(0,n \cdot \sigma^2)$.
By the triangle inequality, \[\dr{\alpha}{Z_n}{Z_n+\Delta} \le \min \left\{ \begin{array}{c} \dr{\infty}{Z_n}{W} + \dr{\alpha}{W}{W+\Delta} + \dr{\infty}{W+\Delta}{Z_n+\Delta} ,\\ \alpha \cdot \left( \sqrt{\dr{*}{Z_n}{W}} + \sqrt{\dr{*}{W}{W+\Delta}} + \sqrt{\dr{*}{W+\Delta}{Z_n+\Delta}} \right)^2 \end{array} \right\}.\]
By Theorem \ref{thm:dgauss-priv}, $\dr{*}{W}{W+\Delta} \le \frac{\Delta^2}{2n\sigma^2}$. 
By Corollary \ref{cor:n-conv}, \[\dr{\pm\infty}{Z_n}{W}=\dr{\pm\infty}{W+\Delta}{Z_n+\Delta} \le 5 \cdot \sum_{k=1}^{n-1} e^{-2\pi^2\sigma^2\frac{k}{k+1}}.\] By Lemma \ref{lem:renyi-misc}, $\dr{*}{Z_n}{W} \le \frac12 (\dr{\pm\infty}{Z_n}{W})^2$. Combining yields the result.
\end{proof}
\fi

Finally, we extend Proposition \ref{prop:dg-sum} to the multidimensional setting using the composition property:
\fi

\begin{prop}[Privacy for Sums of Multidimensional Discrete Gaussians]\label{prop:dg-sum-priv}
Let $\sigma \ge \frac12$.
Let $X_{i,j} \gets \mathcal{N}_{\mathbb{Z}}(0,\sigma^2)$ independently for each $i$ and $j$. Let $X_i = (X_{i,1}, \cdots, X_{i,d}) \in \mathbb{Z}^d$. Let $Z_n=\sum_i^n X_i \in \mathbb{Z}^d$.
\ifarxiv
Then, for all $\Delta \in \mathbb{Z}^d$ and all $\alpha \in [1,\infty)$, \begin{equation}\dr{\alpha}{Z_n}{Z_n+\Delta} \le \min \left\{ \begin{array}{c} \frac{\alpha\|\Delta\|_2^2}{2n\sigma^2} +  \tau \cdot d, \\ \frac{\alpha}{2} \cdot \left(\frac{\|\Delta\|_2^2}{n \sigma^2} + 2 \frac{\|\Delta\|_1}{\sqrt{n} \sigma} \cdot \tau + \tau^2 \cdot d \right) , \\ \frac{\alpha}{2} \cdot \left( \frac{\|\Delta\|_2}{\sqrt{n} \sigma} + \tau \cdot \sqrt{d} \right)^2 \end{array} \right\},\end{equation}
where 
\else
Let
\fi
$\tau := 10 \cdot \sum_{k=1}^{n-1} e^{-2\pi^2\sigma^2\frac{k}{k+1}}$.
An algorithm $M$ that adds $Z_n$ to a query with $\ell_p$ sensitivity $\Delta_p$ satisfies $\frac12\varepsilon^2$-concentrated differential privacy for \begin{equation}\varepsilon = \min \left\{ \begin{array}{c} \sqrt{ \frac{\Delta_2^2}{n\sigma^2} + 2 \tau d} ,\\ \sqrt{\frac{\Delta_2^2}{n \sigma^2} + 2 \frac{\Delta_1}{\sqrt{n} \sigma} \cdot \tau + \tau^2 d},\\ \frac{\Delta_2}{\sqrt{n} \sigma} + \tau \sqrt{d} \end{array} \right\}.\end{equation}
\end{prop}

\ifarxiv
\begin{proof}
This follows from Proposition \ref{prop:dg-sum} and summing over coordinates. Note that before summing we expand \[\sum_i \left(\frac{|\Delta_i|}{\sqrt{n} \sigma} + 10 \cdot \sum_{k=1}^{n-1} e^{-2\pi^2\sigma^2\frac{k}{k+1}}\right)^2 = \sum_i \frac{\Delta_i^2}{n \sigma^2} + 2 \frac{|\Delta_i|}{\sqrt{n} \sigma} \tau + \tau^2 = \frac{\|\Delta\|_2^2}{n \sigma^2} + 2 \frac{\|\Delta\|_1}{\sqrt{n} \sigma} \tau + \tau^2 \cdot d.\] To obtain the third expression we apply the bound $\|\Delta\|_1 \le \sqrt{d} \cdot \|\Delta\|_2$ and complete the square again.
\end{proof}
\fi
Finally, we state a utility bound for the discrete Gaussian.%
\begin{lem}[Utility of the Discrete Gaussian]\label{lem:dg-util}
    Let $X \gets \mathcal{N}_{\mathbb{Z}}(0,\sigma^2)$. Then $\ex{}{X}=0$ and $\var{}{X} = \ex{}{X^2} < \sigma^2$. For all $t \in \mathbb{R}$, $\ex{}{e^{tX}} \le e^{t^2 \sigma^2 /2}$.
\end{lem}

\section{Theoretical Utility Analysis}\label{sec:utility}

We now delve into the accuracy analysis of our algorithm. There are three sources of error that we must account for: (i) discretization via (conditional) randomized rounding, (ii) the noise added for privacy (which depends on the norm of the \emph{discretized} vector), and (iii) the modular clipping operation. We address these concerns one at a time.

\subsection{Randomized Rounding}\label{sec:rr}

In order to apply discrete noise, we must first round the input vectors to the discrete grid. We must analyze the error (both bias and variance) that this introduces, and also ensure that it doesn't increase the sensitivity too much. That is, the rounded vector may have larger norm than the original vector, and we must control this.

We begin by defining the randomized rounding operation:

\begin{defn}[Randomized Rounding]\label{defn:rr}
Let $\gamma>0$ and $d \in \mathbb{N}$. Define $R_\gamma : \mathbb{R}^d \to \gamma\mathbb{Z}^d$ (where $\gamma\mathbb{Z}^d := \{ (\gamma z_1, \gamma z_2, \cdots, \gamma z_d) : z_1, \cdots, z_d \in \mathbb{Z}\} \subset \mathbb{R}^d$) as follows. For $x \in [0,\gamma]^d$, $R_\gamma(x)$ is a product distribution on $\{0,\gamma\}^d$ with mean $x$; that is, indepdently for each $i \in [d]$, we have $\pr{}{R_\gamma(x)_i = 0} = 1-x_i/\gamma$ and $\pr{}{R_\gamma(x)_i = \gamma} = x_i/\gamma$. In general, for $x \in \mathbb{R}^d$, we have $R_\gamma(x) = \gamma\lfloor x/\gamma \rfloor + R_\gamma(x-\gamma\lfloor x/\gamma \rfloor)$; here $\gamma \lfloor x/\gamma \rfloor \in \gamma \mathbb{Z}^d$ is the point $x$ rounded down coordinate-wise to the grid. 
\end{defn}

We first look at how randomized rounding impacts the norm. It is easy to show that \begin{equation}\pr{}{\|R_\gamma(x)-x\|_p \le \gamma \cdot d^{1/p}}=1\label{eqn:rr-worst-case}\end{equation} for all $p \in [1,\infty]$. This bound may be sufficient for many purposes, but, if we relax the probability 1 requirement, we can do better (by constant factors), as demonstrated by the following lemma.

\begin{lem}[Norm of Randomized Rounding]\label{lem:rr-norm}
Let $\gamma > 0$ and $x \in \mathbb{R}^d$. Let $R_\gamma$ be as in Definition \ref{defn:rr}. Then $\ex{}{R_\gamma(x)}=x$, $\ex{}{\|R_\gamma(x)\|_1} = \|x\|_1$, and \begin{equation}\ex{}{\|R_\gamma(x)\|_2^2} = \|x\|_2^2 + \gamma\|y\|_1-\|y\|_2^2 \le \|x\|_2^2 + \frac14 \gamma^2 d,\end{equation} where $y := x - \gamma \lfloor x/\gamma \rfloor \in [0,\gamma]^d$. Furthermore, for any $\beta \in (0,1)$, we have \begin{equation}\pr{}{\|R_\gamma(x)\|_2^2 \le \ex{}{\|R_\gamma(x)\|_2^2} + \sqrt{2\log(1/\beta)} \cdot \gamma \cdot \left(\|x\|_2 + \frac12 \gamma \sqrt{d}\right) } \ge 1-\beta\end{equation}
and \begin{equation}\pr{}{\|R_\gamma(x)\|_1 \le \|x\|_1 + \gamma \cdot \sqrt{\frac12 d \log(1/\beta)} } \ge 1 - \beta.\end{equation}
\end{lem}

\ifarxiv
\begin{proof}
Define $\underline{x}, \overline{x} \in \gamma \mathbb{Z}^d$ by $\underline{x}_i = \gamma \lfloor x_i / \gamma \rfloor$ and $\overline{x}_i = \underline{x}_i + \gamma$. Then $y = x - \underline{x} \in [0,\gamma]^d$.

Fix some $i \in [d]$. By definition, $\pr{}{R_\gamma(x)_i=\underline{x}_i}=1-y_i/\gamma$ and $\pr{}{R_\gamma(x)_i=\underline{x}_i+\gamma}=y_i/\gamma$. Thus $\ex{}{R_\gamma(x)_i} = \underline{x}_i + y_i = x_i$, $\ex{}{|R_\gamma(x)_i|} = |x_i|$, and \[\ex{}{R_\gamma(x)_i^2} = (1-y_i/\gamma) \underline{x}_i^2 + (y_i/\gamma) (\underline{x}_i+\gamma)^2 = (\underline{x}_i+y_i)^2 + \gamma y_i - y_i^2= x_i^2 + \gamma y_i - y_i^2.\]
Note that $ \gamma y_i - y_i^2 \le \gamma (\gamma/2) - (\gamma/2)^2 = \gamma^2/4$ for all values of $y_i \in [0,\gamma]$. Summing over $i \in [d]$ gives $\ex{}{R_\gamma(x)}=x$, $\ex{}{\|R_\gamma(x)\|_1}=\|x\|_1$, and $\ex{}{\|R_\gamma(x)\|_2^2} = \|x\|_2^2 + \gamma \|y\|_1 + \|y\|_2^2 \le \|x\|_2^2 + \gamma^2 d / 4$, as required.

Fix some $i \in [d]$ and $t,\lambda \ge 0$. By Hoeffding's lemma, 
\begin{align*}
    \ex{}{\exp(t \cdot R_\gamma(x)_i^2)} &= (1-y_i/\gamma) \cdot e^{t \underline{x}_i^2} + (y_i/\gamma) \cdot e^{t \cdot \overline{x}_i^2} \\
    &\le \exp\left(t \cdot \ex{}{R_\gamma(x)_i^2} + \frac{t^2}{8} (\overline{x}_i^2 - \underline{x}_i^2)^2\right)\\
    &= \exp\left(t \cdot \ex{}{R_\gamma(x)_i^2} + \frac{t^2}{8} (2\gamma \underline{x}_i + \gamma^2)^2\right)\\
    &\le \exp \left( t \cdot \ex{}{R_\gamma(x)_i^2} + \frac{t^2\gamma^2}{2} \left(x_i^2 + \gamma |x_i| + \frac14 \gamma^2 \right) \right),\\
    \ex{}{\exp(t \cdot \|R_\gamma(x)\|_2^2)} &= \prod_i^d \ex{}{\exp(t \cdot R_\gamma(x)_i^2)} \\
    &\le \exp \left( t \cdot \ex{}{\|R_\gamma(x)\|_2^2} + \frac{t^2\gamma^2}{2} \left( \|x\|_2^2 + \gamma \|x\|_1 + \frac14 \gamma^2 d \right) \right)\\
    &\le \exp \left( t \cdot \ex{}{\|R_\gamma(x)\|_2^2} + \frac{t^2\gamma^2}{2} \left( \|x\|_2^2 + \gamma \sqrt{d} \|x\|_2 + \frac14 \gamma^2 d \right) \right)\\
    &= \exp \left( t \cdot \ex{}{\|R_\gamma(x)\|_2^2} + \frac{t^2\gamma^2}{2} \left( \|x\|_2 + \frac12 \gamma \sqrt{d} \right)^2 \right),\\
    \pr{}{\|R_\gamma(x)\|_2^2 \ge \lambda} &= \pr{}{\exp( t \cdot ( \|R_\gamma(x)\|_2^2 - \lambda ) ) \ge 1}\\
    &\le \ex{}{\exp( t \cdot ( \|R_\gamma(x)\|_2^2 - \lambda ) )}\\
    &\le \exp \left( t \cdot \left(\ex{}{\|R_\gamma(x)\|_2^2} - \lambda \right) + \frac{t^2\gamma^2}{2} \left( \|x\|_2 + \frac12 \gamma \sqrt{d} \right)^2 \right).
\end{align*}
Setting $t = \frac{\lambda - \ex{}{\|R_\gamma(x)\|_2^2}}{\gamma^2 \left( \|x\|_2 + \frac12 \gamma \sqrt{d} \right)^2}$ and $\lambda = \ex{}{\|R_\gamma(x)\|_2^2} + \sqrt{2\gamma^2(\|x\|_2 + \frac12 \gamma \sqrt{d})^2 \log(1/\beta)}$ gives \[\pr{}{\|R_\gamma(x)\|_2^2 \ge \lambda} \le \exp \left( \frac{-(\lambda - \ex{}{\|R_\gamma(x)\|_2^2})^2}{2\gamma^2 \left( \|x\|_2 + \frac12 \gamma \sqrt{d} \right)^2} \right) \le \beta.\]

Fix some $i \in [d]$ and $t,\lambda \ge 0$. Assume, without loss of generality, that $x_i \ge 0$. By Hoeffding's lemma,
\begin{align*}
    \ex{}{\exp\left( t \cdot |R_\gamma(x_i)| \right)} &= (1-y_i/\gamma) \cdot e^{t \cdot \underline{x}_i} + (y_i/\gamma) \cdot e^{t \cdot \overline{x}_i}\\
    &\le \exp\left( t \cdot x_i + \frac{t^2\gamma^2}{8}\right),\\
    \ex{}{\exp( t \cdot \| R_\gamma(x) \|_1) } &\le \exp\left( t \cdot \|x\|_1 + \frac{t^2 \gamma^2}{8} \cdot d \right),\\
    \pr{}{\|R_\gamma(x)\|_1 \ge \lambda} &\le \exp\left( t \cdot \|x\|_1 + \frac{t^2 \gamma^2}{8} \cdot d - t \cdot \lambda \right).
\end{align*}
Setting $t=4\frac{\lambda - \|x\|_1}{\gamma^2 d}$ and $\lambda = \|x\|_1 +\gamma \sqrt{\frac12 d \log(1/\beta)}$ gives \[\pr{}{\|R_\gamma(x)\|_1 \ge \lambda } \le \exp \left( -2\frac{(\lambda-\|x\|_1)^2}{\gamma^2 d} \right) \le \beta.\]

\end{proof}
\fi

\begin{rem}
    The expectation and high probability bounds of Lemma \ref{lem:rr-norm} are only a constant factor better than the worst-case bound \eqref{eqn:rr-worst-case}. Namely, Lemma \ref{lem:rr-norm} gives the bound \[\ex{}{\|R_\gamma(x)\|_2^2} = \|x\|_2^2 + \gamma \|y\|_1 - \|y\|_2^2 \le  \|x\|_2^2 + \frac14\gamma^2 d,\] whereas the worst case bound is \[\|R_\gamma(x)\|_2^2 \le \left( \|x\|_2 + \gamma \sqrt{d} \right)^2 \le 2\|x\|_2^2 + 2\gamma^2 d .\]
    
    Nevertheless, constant factor improvements matter in practical systems. However, hopefully, $\gamma$ is sufficiently small that the increase in norm from randomized rounding is entirely negligible, even if we apply the worst-case bound.
\end{rem}

\begin{rem}
In Lemma \ref{lem:rr-norm}, the inequality $\ex{}{\|R_\gamma(x)\|_2^2} = \|x\|_2^2 + \gamma\|y\|_1-\|y\|_2^2 \le \|x\|_2^2 + \frac14 \gamma^2 d$ is tight when $y=x-\gamma\lfloor x/\gamma\rfloor$ has all entries being $\gamma/2$. However, if $Y \in [0,\gamma]^d$ is uniformly random, then $\ex{}{\gamma\|Y\|_1-\|Y\|_2^2} = \frac16 \gamma^2 d$. Consequently, if we were to round not to $\gamma \mathbb{Z}^d$ but to a randomly translated grid (i.e., $\gamma \mathbb{Z}^d + U$ for a uniformly random $U \in [0,\gamma]^d$), then this error term can be reduced; the shift $U$ can also be released without compromising privacy. We do not explore this direction further.
\end{rem}

Lemma \ref{lem:rr-norm} shows that, with high probability, randomized rounding will not increase the norm too much. 
We could use this directly as the basis of a privacy guarantee -- the probability of the norm being too large would correspond to some kind of privacy failure probability.
Instead what we will do is, if the norm is too large, we simply fix that -- namely, by resampling the randomized rounding procedure. That is, instead of accepting a small probability of privacy failing, we accept a small probability of inaccuracy.

\begin{defn}[Conditional Randomized Rounding]\label{defn:crr}
Let $\gamma>0$ and $d \in \mathbb{N}$ and $G \subset \mathbb{R}^d$. Define $R_\gamma^G : \mathbb{R}^d \to \gamma\mathbb{Z}^d \cap G$ to be $R_\gamma$ conditioned on the output being in $G$. That is, $\pr{}{R_\gamma^G(x)=y}=\pr{}{R_\gamma(x)=y}/\pr{}{R_\gamma(x) \in G}$ for all $y \in \gamma\mathbb{Z}^d \cap G$, where $R_\gamma$ is as in Definition \ref{defn:rr}.
\end{defn}

To implement conditional randomized rounding, we simply re-run $R_\gamma(x)$ again and again until it generates a point in $G$ and then we output that. The expected number of times we must run the randomized rounding to get a point in $G$ is $\pr{}{R_\gamma(x) \in G}^{-1}$. Thus it is important to ensure that $\pr{}{R_\gamma(x) \in G}$ is not too small.

Now we give a lemma that bounds the error of conditional randomized rounding.

\begin{lem}[Error of Conditional Randomized Rounding]\label{lem:crr-util}
Let $\gamma>0$, $x \in \mathbb{R}^d$, and $G \subset \mathbb{R}^d$. Let $1-\beta = \pr{}{R_\gamma(x) \in G}>0$. Then
\begin{equation}\left\|\ex{}{R_\gamma^G(x)}-x\right\|_2 \le \frac{\beta \cdot \gamma \cdot \sqrt{d}}{1-\beta}\end{equation} and \begin{equation}\ex{}{\left\| R_\gamma^G(x) - \ex{}{R_\gamma^G(x)} \right\|_2^2 } \le \ex{}{\left\| R_\gamma^G(x) - x \right\|_2^2 } \le \frac{\gamma^2 d}{4(1-\beta)}.\end{equation}
For all $t \in \mathbb{R}^d$,
\begin{equation}\ex{}{\exp(\langle t , R_\gamma^G(x)-x \rangle )} \le \frac{\exp \left( \frac{\gamma^2}{8} \cdot \|t\|_2^2 \right)}{1-\beta}.\end{equation}
\end{lem}
\ifarxiv 
\begin{proof}
For an arbitrary random variable $X$ and nontrivial event $E$, we have \begin{equation}\ex{}{X} = \ex{}{X|E}\pr{}{E}+\ex{}{X|\overline{E}}\pr{}{\overline{E}},\end{equation} which rearranges to give \begin{equation}\ex{}{X|E} = \frac{\ex{}{X}-\ex{}{X|\overline{E}}\pr{}{\overline{E}}}{\pr{}{E}}\end{equation} and \begin{equation}\ex{}{X|E}-\ex{}{X} = \frac{\pr{}{\overline{E}}}{\pr{}{E}}\left(\ex{}{X}-\ex{}{X|\overline{E}}\right).\end{equation}
Thus 
\begin{align*}
    \left\|\ex{}{R_\gamma^G(x)}-x\right\|_2 &= \left\|\ex{}{R_\gamma(x)| R_\gamma(x) \in G}-\ex{}{R_\gamma(x)}\right\|_2 \\
    &= \frac{\pr{}{R_\gamma(x) \notin G}}{\pr{}{R_\gamma(x) \in G}}\left\|\ex{}{R_\gamma(x)}-\ex{}{R_\gamma(x)| R_\gamma(x) \notin G}\right\|_2\\
    &= \frac{\beta}{1-\beta} \| x - x^* \|_2,
\end{align*}
where $x^* \in \gamma \lfloor x / \gamma \rfloor + [0,\gamma]^d$ and, hence, $\|x-x^*\|_2 \le \gamma \cdot \sqrt{d}$.

Next, we have
\begin{align*}
    \ex{}{\left\| R_\gamma^G(x) - x \right\|_2^2} &= \frac{\ex{}{\left\| R_\gamma(x) - x \right\|_2^2} - \ex{}{\left\| R_\gamma(x) - x \right\|_2^2| R_\gamma(x) \notin G} \pr{}{R_\gamma(x) \notin G}}{\pr{}{R_\gamma(x) \in G}}\\
    &\le \frac{\ex{}{\left\| R_\gamma(x) - x \right\|_2^2}}{\pr{}{R_\gamma(x) \in G}}\\
    &= \frac{\gamma\|y\|_1-\|y\|_2^2}{1-\beta} \tag{Lemma \ref{lem:rr-norm}}\\
    &\le \frac{\gamma^2 d}{4(1-\beta)}.
\end{align*}
Note that the apply the bias variance decomposition: Since $\ex{}{R_\gamma(x)}=x$, we have $\ex{}{\left\| R_\gamma(x) \right\|_2^2} = \ex{}{\left\| R_\gamma(x) - x \right\|_2^2} + \|x\|_2^2$. Similarly, \[\ex{}{\left\| R_\gamma^G(x) - x \right\|_2^2} = \ex{}{\left\| R_\gamma^G(x) - \ex{}{R_\gamma^G(x)} \right\|_2^2} + \left\| \ex{}{R_\gamma^G(x)} - x \right\|_2^2 \ge \ex{}{\left\| R_\gamma^G(x) - \ex{}{R_\gamma^G(x)} \right\|_2^2}.\]

By Hoeffding's lemma, since $R_\gamma(x) \in \gamma \lfloor x/\gamma \rfloor + \{0,\gamma\}^d$ and is a product distribution with mean $x$, we have \begin{equation}\forall t \in \mathbb{R}^d ~~~~~ \ex{}{\exp(\langle t , R_\gamma(x)-x \rangle )} \le \exp \left( \frac{\gamma^2}{8} \cdot \|t\|_2^2 \right). \end{equation}
Thus \[\forall t \in \mathbb{R}^d ~~~~~ \ex{}{\exp(\langle t , R_\gamma^G(x)-x \rangle )} \le \frac{\ex{}{\exp(\langle t , R_\gamma(x)-x \rangle )}}{\pr{}{R_\gamma(x) \in G}} \le \frac{\exp \left( \frac{\gamma^2}{8} \cdot \|t\|_2^2 \right)}{1-\beta}. \]
\end{proof}
\fi

We summarize the results of this section. First we give a proposition for a single instance of randomized rounding (this combines Lemma \ref{lem:rr-norm} and \ref{lem:crr-util}).

\begin{prop}[Properties of Randomized Rounding]
Let $\beta \in [0,1)$, $\gamma > 0$, and $x \in \mathbb{R}^d$. Let \begin{equation}\Delta_2^2 := \min \left\{ \begin{array}{c} \|x\|_2^2 + \frac14 \gamma^2 d + \sqrt{2\log(1/\beta)} \cdot \gamma \cdot \left(\|x\|_2 + \frac12 \gamma \sqrt{d}\right) ,\\ \left( \|x\|_2 + \gamma \sqrt{d} \right)^2 \end{array} \right\}\end{equation} and $G := \left\{ y \in \mathbb{R}^d : \|y\|_2^2 \le \Delta_2^2 \right\}$. Let $R_\gamma(x)$ and $R_\gamma^G(x)$ be as in Definitions \ref{defn:rr} and \ref{defn:crr}.

Then $\pr{}{R_\gamma(x) \in G} \ge 1-\beta$ and, consequently, the following hold.
\begin{align}
    \left\|\ex{}{R_\gamma^G(x)}-x\right\|_2 &\le \frac{\beta \cdot \gamma \cdot \sqrt{d}}{1-\beta}.\\
    \ex{}{\left\| R_\gamma^G(x) - \ex{}{R_\gamma^G(x)} \right\|_2^2 } &\le \frac{\gamma^2 d}{4(1-\beta)}.\\
    \forall t \in \mathbb{R}^d ~~~~~ \ex{}{\exp(\langle t , R_\gamma^G(x)-x \rangle )} &\le \frac{\exp \left( \frac{\gamma^2}{8} \cdot \|t\|_2^2 \right)}{1-\beta}.\\
    \pr{}{\|R_\gamma^G(x)\|_2^2 \le \Delta_2^2} &= 1 = \pr{}{R_\gamma^G(x) \in \gamma \mathbb{Z}^d}.
\end{align}
\end{prop}

Now we give a proposition for sums of randomized roundings.

\begin{prop}[Randomized Rounding \& Sums]\label{prop:rr-sum}
Let $\beta \in [0,1)$, $\gamma > 0$, and $x_1, \cdots, x_n \in \mathbb{R}^d$. Suppose $ \|x_i\|_2 \le c$ for all $i \in [n]$. Let \begin{equation}\Delta_2^2 := \min \left\{ \begin{array}{c} c^2 + \frac14 \gamma^2 d + \sqrt{2\log(1/\beta)} \cdot \gamma \cdot \left(c + \frac12 \gamma \sqrt{d}\right) ,\\ \left( c + \gamma \sqrt{d} \right)^2 \end{array} \right\}\end{equation} and $G := \left\{ y \in \mathbb{R}^d : \|y\|_2^2 \le \Delta_2^2 \right\}$. Let $R_\gamma^G$ be as in Definition \ref{defn:crr}.

Then the following hold.
\begin{align}
    \left\|\ex{}{\sum_i^n R_\gamma^G(x_i)}-\sum_i^n x_i\right\|_2 &\le \frac{\beta \cdot \gamma \cdot \sqrt{d} \cdot n}{1-\beta}.\\
    \ex{}{\left\| \sum_i^n R_\gamma^G(x_i) - \ex{}{\sum_i^n R_\gamma^G(x_i)} \right\|_2^2 } &\le \frac{\gamma^2 \cdot d \cdot n}{4(1-\beta)}.\\
    \ex{}{\left\| \sum_i^n R_\gamma^G(x_i) - \sum_i^n x_i \right\|_2^2 } &\le \frac{\gamma^2 \cdot d \cdot n}{4(1-\beta)} + \left( \frac{\beta}{1-\beta} \gamma \sqrt{d} n \right)^2.\label{eq:rr-sum-sqerr}\\
    \forall t \in \mathbb{R}^d ~~~~~ \ex{}{\exp\left(\left\langle t , \sum_i^n R_\gamma^G(x_i) - \sum_i^n x_i \right\rangle \right)} &\le \frac{\exp \left( \frac{\gamma^2}{8} \cdot \|t\|_2^2 \cdot n \right)}{(1-\beta)^{n}}.\\
    \pr{}{\forall i ~~ \|R_\gamma^G(x_i)\|_2^2 \le \Delta_2^2} &= 1 = \pr{}{\forall i ~~ R_\gamma^G(x_i) \in \gamma \mathbb{Z}^d}.
\end{align}
\end{prop}

\begin{rem}
    Proposition \ref{prop:rr-sum} provides some guidance on how to set the parameter $\beta$. We have the mean squared error bound \eqref{eq:rr-sum-sqerr} \[\ex{}{\left\| \sum_i^n R_\gamma^G(x_i) - \sum_i^n x_i \right\|_2^2 } \le \frac{\gamma^2 \cdot d \cdot n \cdot (1 - \beta + 4 \beta^2 n)}{4(1-\beta)^2}.\]
    If we set $\beta \approx 1/\sqrt{n}$, then the bias and variance terms in the bound are of the same order. 
    Setting $\beta$ too small would needlessly increase the sensitivity $\Delta_2$. And we see that there is little value in setting $\beta \ll 1/\sqrt{n}$. So the theory suggests setting $\beta \approx 1/\sqrt{n}$.
    
    However, we emphasize that this is a worst-case upper bound on the error and it is likely that, in practice, the error would likely be considerably less. Thus it is justifiable to set $\beta$ to be considerably larger -- e.g., $\beta = e^{-1/2}$ -- and simply hope for the best in terms of accuracy.
\end{rem}

\begin{rem}
    Proposition \ref{prop:rr-sum} covers the error introduced by discretization. Obviously, reducing the granularity $\gamma$ will reduce the discretization error. However, this comes at a cost in communication, so the choice of this parameter will need to be carefully made.
\end{rem}

In Section \ref{sec:dg} we have covered the noise that is injected to preserve privacy and in Section \ref{sec:rr} we have covered the error introduced by discretizing the data.
Now we state a result that combines these.

\begin{prop}[Randomized Rounding + Discrete Gaussian]\label{prop:dg-rr-sum}
Let $\beta \in [0,1)$, $\sigma^2 \ge \frac12 \gamma > 0$, and $c>0$.
Let 
\begin{align}
    \Delta_2^2 &:= \min \left\{ \begin{array}{c} c^2 + \frac14 \gamma^2 d + \sqrt{2\log(1/\beta)} \cdot \gamma \cdot \left(c + \frac12 \gamma \sqrt{d}\right) ,\\ \left( c + \gamma \sqrt{d} \right)^2 \end{array} \right\},\\
    G &:= \left\{ y \in \mathbb{R}^d : \|y\|_2^2 \le \Delta_2^2 \right\},\\
    \tau &:= 10 \cdot \sum_{k=1}^{n-1} e^{-2\pi^2\frac{\sigma^2}{\gamma^2} \cdot \frac{k}{k+1}},\\
    \varepsilon &:= \min \left\{ \begin{array}{c} \sqrt{ \frac{\Delta_2^2}{n\sigma^2} + 2 \tau d} ,\\ \frac{\Delta_2}{\sqrt{n} \sigma} + \tau \sqrt{d} \end{array} \right\}.
\end{align} 
Let $R_\gamma^G$ be as in Definition \ref{defn:crr}. Define a randomized algorithm $A : (\mathbb{R}^d)^n \to \gamma \mathbb{Z}^d$ by\footnote{Note that $\min\left\{1, \frac{c}{\|x_i\|_2} \right\} \cdot x_i$ is simply $x_i$ with its norm clipped to $c$.} 
\begin{equation}
    A(x) = \sum_i^n R_\gamma^G\left(\min\left\{1, \frac{c}{\|x_i\|_2} \right\} \cdot x_i\right) + \gamma \cdot Y_i,
\end{equation}
where $Y_1, \cdots, Y_n \in \mathbb{Z}^d$ are independent random vectors with each entry drawn independently from $\mathcal{N}_{\mathbb{Z}}(0,\sigma^2/\gamma^2)$.

Then $A$ satisfies $\frac12 \varepsilon^2$-concentrated differential privacy.\footnote{Not that this is with respect to the addition or removal of an element, not replacement. To keep $n$ fixed, we would need addition/removal to be defined to simply zero-out the relevant vectors.} 

Let $x_1, \cdots, x_n \in \mathbb{R}^d$ with $ \|x_i\|_2 \le c$ for all $i \in [n]$.
Then the following hold.
\begin{align}
    \left\|\ex{}{A(x)}-\sum_i^n x_i\right\|_2 &\le \frac{\beta \cdot \gamma \cdot \sqrt{d} \cdot n}{1-\beta}.\\
    \ex{}{\left\| A(x) - \ex{}{A(x)} \right\|_2^2 } &\le \frac{\gamma^2 \cdot d \cdot n}{4(1-\beta)} + n \cdot d \cdot \sigma^2.\\
    \ex{}{\left\| A(x) - \sum_i^n x_i \right\|_2^2 } &\le \frac{\gamma^2 \cdot d \cdot n}{4(1-\beta)} + \left( \frac{\beta}{1-\beta} \gamma \sqrt{d} n \right)^2 +  n \cdot d \cdot \sigma^2\\
    \forall t \in \mathbb{R}^d ~~~ \ex{}{\exp\left(\left\langle t , A(x) - \sum_i^n x_i \right\rangle \right)} &\le \frac{\exp \left( \left(\frac{\gamma^2}{8} + \frac{\sigma^2}{2} \right) \cdot \|t\|_2^2 \cdot n \right)}{(1-\beta)^{n}}.
\end{align}
\end{prop}

\subsection{Flattening}

It is possible that the inputs $x_i$ and the sum $\bar{x} = \sum_i^n x_i$ are very heavily concentrated on one coordinate. This is a bad case, as the modular clipping will create a very large error, unless we use a very large modulus.
To avoid this problem we will ``flatten'' the inputs as a pre-processing step (which is inverted by the server at the end of the protocol).

Specifically, our goal is to pre-process the inputs $x_1, \cdots, x_n$ so that $\sum_i x_i \in [a,b]^d$ and then at the end we can undo the pre-processing to obtain the original value. Here $[a,b]$ is the range where modular arithmetic does not cause errors.

To flatten the vectors we will transform the data points by multiplying them by a (random) matrix $U \in \mathbb{R}^{d \times d}$; at the end we multiply by $U^{-1}$ to undo this operation. We will take $U$ to be a unitary matrix so that $U^{-1} = U^T$.\footnote{We take $U$ to be a square matrix, but it is in general possible to also increase the dimension during this pre-processing step. We also could extend beyond unitary matrices to invertible (and well-conditioned) matrices. } 

\begin{rem}
    The flattening matrix $U$ is \emph{shared randomness} -- that is, the server and all the clients must have access to this matrix. Fortunately, the differential privacy guarantee does not depend on this randomness remaining hidden; thus $U$ can be published, and we do not need to worry about the privacy adversary having access to it. 
\end{rem}

The specific property that we want the flattening matrix $U$ to satisfy is that, for all $x \in \mathbb{R}^d$ and all $i \in [d]$, the value $(Ux)_i \in \mathbb{R}$ has a subgaussian distribution (determined by the randomness of $U$) with variance proxy $O(\|x\|_2^2/{d})$.

There are many possibilities for this flattening transformation. A natural option is for $U$ to be a random unitary matrix or rotation matrix. This would attain the desired property:
\begin{lem}\label{lem:rotation}
    Let $x \in \mathbb{R}^d$ be fixed.
    Let $U \in \mathbb{R}^{d \times d}$ be a uniformly random unitary matrix or rotation matrix, so that $Ux$ is a uniformly random vector with norm $\|x\|_2$.
    Then $\ex{}{e^{t(Ux)_i}} \le e^{t^2 \|x\|_2^2/2d}$ for all $t \in \mathbb{R}$ and all $i \in [d]$.
\end{lem}

\ifarxiv 
\begin{proof}
Let $Y=((Ux)_i/\|x\|_2+1)/2 \in [0,1]$. Then $Y$ follows a $\mathsf{Beta}((d-1)/2,(d-1)/2)$ distribution \cite{SE-unit-vector-beta}. 
Thus $\ex{}{e^{t(Y-1/2)}} \le e^{t^2/8}$ for all $t \in \mathbb{R}$ \cite{Beta-subgaussian}. The result follows.
\end{proof}
\fi
Unfortunately, a random rotation or unitary matrix has several downsides: Generating a random unitary or rotation matrix is itself a non-trivial task \cite{Sample-unitary}. Even storing such a matrix in memory and performing the required matrix-vector multiplication could be prohibitive -- it would take $\Theta(d^2)$ time and space where we seek algorithms that run in $\tilde{O}(d)$ time and space.

Instead our approach is to first randomize the signs of the entries of $x$ and then multiply by a matrix with small entries. This attains the desired guarantee:
\begin{lem}\label{lem:flatHD}
    Let $H \in [-\sqrt{\rho/d},\sqrt{\rho/d}]^{d \times d}$ be a fixed unitary matrix.
    Let $D \in \{0,-1,+1\}^{d \times d}$ be a diagonal matrix with random signs on the diagonal.
    Fix $x \in \mathbb{R}^d$ and $i \in [d]$. Let $Y=(HDx)_i \in \mathbb{R}$. Then $\ex{}{e^{tY}} \le e^{t^2\|x\|_2^2 \rho /2d}$ for all $t \in \mathbb{R}$.
\end{lem}
\begin{proof}
We have $Y=(HDx)_i = \sum_j^d H_{i,j} D_{j,j} x_j$ and, by independence, $\ex{}{e^{tY}} = \prod_j^d \ex{}{e^{t H_{i,j} D_{j,j} x_j}}$. Since $H_{i,j} \in [-\sqrt{\rho/d}, \sqrt{\rho/d}]$ and $D_{j,j} \in \{-1,+1\}$ is uniformly random, we have $\ex{}{e^{t H_{i,j} D_{j,j} x_j}} \le e^{t^2 (\sqrt{\rho/d})^2 x_j^2 / 2}$ by Hoeffding's lemma.
\end{proof}

Generating the required random signs is easy. (We can use a pseudorandom generator and a small shared random seed.) We just need a unitary matrix $H \in [-\sqrt{\rho/d},\sqrt{\rho/d}]^{d \times d}$ with $\rho$ as small as possible.\footnote{A lower bound of $\rho \ge 1$ applies as $H$ is unitary -- $H^TH=I$, so $1 = (H^TH)_{1,1} = \sum_i^d H_{i,1}^2 \le d (\sqrt{\rho/d})^2 = \rho$.} And we want $H$ to be easy to work with -- specifically, we want efficient algorithms to compute the matrix-vector product $Hx$ and its inverse $H^T x$ for arbitrary $x \in \mathbb{R}^d$.

Walsh-Hadamard matrices are ideal for $H$ (after scaling appropriately). They attain the optimal $\rho = 1$ and the fast Walsh-Hadamard transform can compute the matrix-vector products in $O(d \log d)$ operations. This is what we use in our experiments. Formally, the Walsh-Hadamard matrices are defined recursively as follows: 
\begin{equation}
    H_{2^0} = (1), ~~~~~~~~~~ \forall k \ge 0 ~~~~~ H_{2^{k+1}} = \frac{1}{\sqrt{2}} \left( \begin{array}{cc} H_{2^k} & H_{2^k} \\ H_{2^k} & -H_{2^k} \end{array} \right) \in \left\{\frac{-1}{\sqrt{2^{k+1}}} , \frac{1}{\sqrt{2^{k+1}}}\right\}^{2^{k+1} \times 2^{k+1}}.\label{eqn:hadamard}
\end{equation}

The only downside of Walsh-Hadamard matrices is that they require the dimension $d$ to be a power of $2$. We can pad the input vectors with zeros to ensure this. However, in the worst case, padding may nearly double the dimension $d$, which correspondingly slows down our algorithm. (E.g., if $d=2^k+1$, then we must pad to dimension $d=2^{k+1}$.)

To avoid or reduce padding, there are several solutions: 
\begin{itemize}
    \item The first, and simplest, solution is to use the discrete cosine transform as the matrix $H_d' \in \mathbb{R}^{d \times d}$, which is defined by
    \begin{equation}
        \forall d \in \mathbb{N} ~~ \forall i, j \in [d] ~~~~~~~~~~ H_d'(i,j) = \sqrt{\frac{2}{d}} \cdot \cos\left(\frac{\pi}{d}\cdot i \cdot (j-\frac12)\right).
    \end{equation}
    Such matrices exist for any dimension $d$ and the required matrix-vector products can still be computed in $O(d \log d)$ time. However, they attain a slightly suboptimal subgaussian flatness parameter of $\rho = 2$.

    \item The second solution is to use more general Hadamard matrices. It is conjectured that for all $d \in \{1,2\}\cup \{4\ell : \ell \in \mathbb{N}\}$ there exist unitary matrices $H_d$ in $\{-1/\sqrt{d},1/\sqrt{d}\}^{d \times d}$. This conjecture would allow us to do very little padding (adding at most three extra coordinates to reach the next multiple of $4$), but we do not have a proof of this conjecture, much less efficient algorithms for computing with these matrices. 
    
    Fortunately, there are explicit constructions of Hadamard matrices of many sizes which also allow efficient matrix-vector computations. By considering sizes other than powers of $2$, we can significantly reduce the required amount of padding.
    
    For example, we can generalize the Kronecker product construction \eqref{eqn:hadamard} to dimension $d=12 \cdot 2^k$ for integers $k \ge 0$:
    \begin{scriptsize}
    \begin{equation}
        \!\! H_{12 \cdot 2^k} \! = \! \frac{1}{\sqrt{12}} \! \left( \begin{array}{cccccccccccc}
            H_{2^k} & -H_{2^k} & -H_{2^k} & -H_{2^k} & -H_{2^k} & -H_{2^k} & -H_{2^k} & -H_{2^k} & -H_{2^k} & -H_{2^k} & -H_{2^k} & -H_{2^k} \\
            H_{2^k} & H_{2^k} & -H_{2^k} & H_{2^k} & -H_{2^k} & -H_{2^k} & -H_{2^k} & H_{2^k} & H_{2^k} & H_{2^k} & -H_{2^k} & H_{2^k} \\
            H_{2^k} & H_{2^k} & H_{2^k} & -H_{2^k} & H_{2^k} & -H_{2^k} & -H_{2^k} & -H_{2^k} & H_{2^k} & H_{2^k} & H_{2^k} & -H_{2^k} \\
            H_{2^k} & -H_{2^k} & H_{2^k} & H_{2^k} & -H_{2^k} & H_{2^k} & -H_{2^k} & -H_{2^k} & -H_{2^k} & H_{2^k} & H_{2^k} & H_{2^k} \\
            H_{2^k} & H_{2^k} & -H_{2^k} & H_{2^k} & H_{2^k} & -H_{2^k} & H_{2^k} & -H_{2^k} & -H_{2^k} & -H_{2^k} & H_{2^k} & H_{2^k} \\
            H_{2^k} & H_{2^k} & H_{2^k} & -H_{2^k} & H_{2^k} & H_{2^k} & -H_{2^k} & H_{2^k} & -H_{2^k} & -H_{2^k} & -H_{2^k} & H_{2^k} \\
            H_{2^k} & H_{2^k} & H_{2^k} & H_{2^k} & -H_{2^k} & H_{2^k} & H_{2^k} & -H_{2^k} & H_{2^k} & -H_{2^k} & -H_{2^k} & -H_{2^k} \\
            H_{2^k} & -H_{2^k} & H_{2^k} & H_{2^k} & H_{2^k} & -H_{2^k} & H_{2^k} & H_{2^k} & -H_{2^k} & H_{2^k} & -H_{2^k} & -H_{2^k} \\
            H_{2^k} & -H_{2^k} & -H_{2^k} & H_{2^k} & H_{2^k} & H_{2^k} & -H_{2^k} & H_{2^k} & H_{2^k} & -H_{2^k} & H_{2^k} & -H_{2^k} \\
            H_{2^k} & -H_{2^k} & -H_{2^k} & -H_{2^k} & H_{2^k} & H_{2^k} & H_{2^k} & -H_{2^k} & H_{2^k} & H_{2^k} & -H_{2^k} & H_{2^k} \\
            H_{2^k} & H_{2^k} & -H_{2^k} & -H_{2^k} & -H_{2^k} & H_{2^k} & H_{2^k} & H_{2^k} & -H_{2^k} & H_{2^k} & H_{2^k} & -H_{2^k} \\
            H_{2^k} & -H_{2^k} & H_{2^k} & -H_{2^k} & -H_{2^k} & -H_{2^k} & H_{2^k} & H_{2^k} & H_{2^k} & -H_{2^k} & H_{2^k} & H_{2^k} 
        \end{array} \right).
    \end{equation}
    \end{scriptsize}
    The addition of this construction alone is sufficient to reduce the worst case for padding from a factor of $2$ to a factor of $1.5$ -- now $2^k+1$ can be padded to $12 \cdot 2^{k-3} = 1.5 \cdot 2^k$. The other desirable properties of the Hadamard matrices are also retained.

    \item A third solution is to move from the reals to complex numbers and use the discrete Fourier transform. Our real vector of length $d$ can be encoded as a complex vector of length $d/2$ (two real entries become the real and imaginary components of one complex entry). Instead of $D$ being a diagonal matrix with random signs, the diagonal entries are $e^{\sqrt{-1}\cdot\theta}$ for a uniformly random $\theta \in [0,2\pi)$. (In fact, it suffices to have $\theta \in \{0,\pi/2,\pi,3\pi/2\}$ uniform. This only requires one bit of shared randomness per coordinate.\footnote{Note that $\theta \in \{0,\pi/2,\pi,3\pi/2\}$ corresponds to $e^{i\theta} \in \{1,i,-1,-i\}$ and, in Equation \ref{eqn:Dtheta}, to $R_\theta \in \left\{ \left(\begin{array}{cc}1&0\\0&1\end{array}\right) , \left(\begin{array}{cc}0&-1\\1&0\end{array}\right) , \left(\begin{array}{cc}-1&0\\0&-1\end{array}\right) , \left(\begin{array}{cc}0&1\\-1&0\end{array}\right) \right\}$.}) Then $H$ is the discrete Fourier transform matrix. This gives us a complex vector of length $d/2$ that can be decoded back to a real vector of length $d$. This transformation is unitary and linear and attains the optimal subgaussian flatness constant $\rho = 1$ (i.e., matches the guarantee of a random rotation or unitary matrix from Lemma \ref{lem:rotation}). The only requirement is that the dimension must be even -- i.e., we must pad at most one zero. %
    
    If we wish to avoid thinking about complex numbers, the complex numbers can be replaced with $2 \times 2$ rotation matrices. That is 
    \begin{equation}
        H_{2d}'' =  \frac{1}{\sqrt{d}} \left( \begin{array}{cccccc}
            W^0 & W^0 & W^0 & W^0 & \cdots & W^0 \\
            W^0 & W^1 & W^2 & W^3 & \cdots & W^d \\
            W^0 & W^2 & W^4 & W^6 & \cdots & W^{2d} \\
            W^0 & W^3 & W^6 & W^9 & \cdots & W^{3d} \\
            \vdots & \vdots & \vdots & \vdots & \ddots & \vdots \\
            W^0 & W^d & W^{2d} & W^{3d} & \cdots & W^{d^2}
        \end{array} \right) \in \mathbb{R}^{2d \times 2d},
        \label{eqn:Hfourier}
    \end{equation} 
    where $W = \left( \begin{array}{cc} \cos(2\pi/d) & -\sin(2\pi/d) \\ \sin(2\pi/d) & \cos(2\pi/d) \end{array} \right) \in \mathbb{R}^{2 \times 2}$, and
    \begin{equation}
        D_\Theta = \left( \begin{array}{ccccc}
            R_{\theta_1} & 0 & 0 & \cdots & 0 \\
            0 & R_{\theta_2} & 0 & \cdots & 0 \\
            0 & 0 & R_{\theta_3} & \cdots & 0 \\
            \vdots & \vdots & \vdots & \ddots & \vdots \\
            0 & 0 & 0 & \cdots & R_{\theta_d}
        \end{array} \right) \in \mathbb{R}^{2d \times 2d},
        \label{eqn:Dtheta}
    \end{equation}
    where $\Theta \in [0,2\pi)^d$ and $R_\theta = \left( \begin{array}{cc} \cos \theta & -\sin \theta \\ \sin \theta & \cos \theta \end{array} \right) \in \mathbb{R}^{2 \times 2}$.
    \begin{prop}
    Let $d \in \mathbb{N}$ be even and let $\Theta$ be uniformly random in either $[0,2\pi)^{d/2}$ or $\{0,\pi/2,\pi,3\pi/2\}^{d/2}$.
    Define $U = H_{d}'' D_\Theta \in \mathbb{R}^{d \times d}$ where $H_d''$ and $D_\Theta$ are given in Equations \ref{eqn:Hfourier} and \ref{eqn:Dtheta}.
    Then $U$ is unitary and $\ex{}{\exp(t(Ux)_i)} \le \exp(t^2 \|x\|_2^2 / 2d)$ for all $x \in \mathbb{R}^d$, all $i \in [d]$, and all $t \in \mathbb{R}$.
    \end{prop}
    \begin{proof}
    The fact that $U$ is unitary follows from the fact that both $H_d''$ and $D_\Theta$ are. Since $R_{\theta_i} = \left( \begin{array}{cc} \cos \theta_i & -\sin \theta_i \\ \sin \theta_i & \cos \theta_i \end{array} \right) \in \mathbb{R}^{2 \times 2}$ is unitary, this implies $D_\Theta$ is unitary. The matrix $H_{d}''$ is a block matrix with the block in row $i \in [d/2]$ and column $j \in [d/2]$ being $W^{(i-1)(j-1)}$ for $W = R_{4\pi/d}$. Then $H_d'' (H_d'')^T$ is also a block matrix. The block in row $i \in [d/2]$ and column $j \in [d/2]$ is
    \begin{align*}
        (H_d'' (H_d'')^T)_{2i-1:2i , 2j-1:2j} &= \frac2d \sum_{k=1}^{d/2} W^{(i-1)(k-1)} (W^{(k-1)(j-1)})^T \\
        &= \frac2d \sum_{k=1}^{d/2} W^{(i-1)(k-1)-(k-1)(j-1)} \\
        &= \frac2d \left\{ \begin{array}{cl} \frac{d}{2} I & \text{ if } i=j \\ 0 & \text{ if } i \ne j  \end{array}\right..
    \end{align*}
    This follows from the fact that $W$ is unitary (i.e., $W^T = W^{-1}$), the fact that $W^{d/2} = I$, and the fact that, for any $\ell\in\mathbb{Z}$ that is \emph{not} a multiple of $d/2$, $W^\ell-I$ is nonsingular and, hence, $\sum_{k=1}^{d/2} W^{\ell\cdot(k-1)} = 0$. 
    
    Fix $x \in \mathbb{R}^d$ and $i \in [d/2]$. Now, by the properties of rotation matrices, 
    \begin{align*}
        \left(\begin{array}{c} (Ux)_{2i-1} \\ (Ux)_{2i} \end{array}\right) &= \sqrt{\frac{2}{d}}\sum_{j=1}^{d/2} W^{(i-1)(j-1)} R_{\theta_j} \left(\begin{array}{c} x_{2i-1} \\ x_{2i} \end{array}\right) \\
        &= \sqrt{\frac{2}{d}}\sum_{j=1}^{d/2} R_{4\pi(i-1)(j-1)/d + \theta_j + \tilde\theta_i } \left(\begin{array}{c} \sqrt{x_{2i-1}^2 + x_{2i}^2} \\ 0 \end{array}\right),
    \end{align*}
    where $\tilde\theta_i \in [0,2\pi)$ is such that $\cos \tilde\theta_i = \frac{x_{2i-1}}{\sqrt{x_{2i-1}^2 + x_{2i}^2}}$ and $\sin \tilde\theta_i = \frac{x_{2i}}{\sqrt{x_{2i-1}^2 + x_{2i}^2}}$.
    For $j \in [d/2]$, define $\hat\theta_{i,j} =  4\pi(i-1)(j-1)/d + \tilde\theta_i $.
    Then $(Ux)_{2i-1} = \sqrt{\frac2d} \sum_{j=1}^{d/2} \sqrt{x_{2i-1}^2 + x_{2i}^2} \cdot \cos(\theta_j + \hat\theta_{i,j}) $ and $(Ux)_{2i} = \sqrt{\frac2d} \sum_{j=1}^{d/2} \sqrt{x_{2i-1}^2 + x_{2i}^2} \cdot \sin(\theta_j + \hat\theta_{i,j}) $.
    Fix $t \in \mathbb{R}$. If $\Theta \in [0,2\pi)^{d/2}$ follows a product distribution, then 
    \begin{equation}
        \ex{\Theta}{\exp(t(Ux)_{2i-1})} = \prod_{j=1}^{d/2} \ex{\theta_j}{\exp\left(t\sqrt{\frac2d (x_{2i-1}^2+x_{2i}^2)} \cos(\theta_j + \hat\theta_{i,j}) \right)} \label{eq:2i-1}
    \end{equation}
    and, similarly,
    \begin{equation}
        \ex{\Theta}{\exp(t(Ux)_{2i})} = \prod_{j=1}^{d/2} \ex{\theta_j}{\exp\left(t\sqrt{\frac2d (x_{2i-1}^2+x_{2i}^2)} \cos(\theta_j + \hat\theta_{i,j} -\pi/2)\right)},\label{eq:2i}
    \end{equation} as $\sin(\psi) = \cos(\psi-\pi/2)$ for all $\psi \in \mathbb{R}$.
    If we can show that $\ex{\theta_j}{\exp(\lambda \cdot \cos(\theta_j + \psi)} \le \exp(\lambda^2/4)$ for all $\lambda,\psi \in \mathbb{R}$ and all $j$, then we are done. Lemma \ref{lem:theta4} covers the case where $\theta_j$ is uniform on $\{0,\pi/2,\pi,3\pi/2\}$ and Lemma \ref{lem:theta-cts} covers the case where it is uniform on $[0,2\pi)$.
    Then the right sides of both Equations \ref{eq:2i-1} and \ref{eq:2i} become
    \begin{align*}
        &\le \prod_{j=1}^{d/2} {\exp\left(\left( t\sqrt{\frac2d (x_{2i-1}^2+x_{2i}^2)} \right)^2/4\right)}\\
        &= \exp\left(\sum_{j=1}^{d/2} t^2 \frac2d (x_{2i-1}^2+x_{2i}^2) \frac14\right)\\
        &= \exp(t^2\|x\|_2^2/2d).
    \end{align*}
    \end{proof}
    \begin{lem}\label{lem:theta4}
        Let $\theta \in \{0,\pi/2,\pi,3\pi/2\}$ be uniformly random. Let $\lambda, \psi \in \mathbb{R}$ be fixed. Then $\ex{\theta}{\exp(\lambda\cdot\cos(\theta+\psi)}\le\exp(\lambda^2/4)$.
    \end{lem}    
    \begin{proof}
        Let $x=\lambda\cos \psi$ and $y=\lambda\cos(\psi+\pi/2)$. Then $\lambda\cos(\psi+\pi)=-x$ and $\lambda\cos(\psi+3\pi/2)=-y$. Note $x^2+y^2=\lambda^2$. Thus 
        \begin{align*}
            \ex{\theta}{\exp(\lambda\cdot\cos(\theta+\psi)} &= \frac{e^x + e^{-x} + e^y + e^{-y} }{4}\\
            &= 1 + \sum_{k=1}^\infty \frac{x^{2k}+y^{2k}}{2 \cdot (2k)!}\\
            &\le 1 + \sum_{k=1}^\infty \frac{\lambda^{2k}}{4^k \cdot k!}\\
            &= \exp(\lambda^2/4).
        \end{align*}
        The inequality follows from the fact that $x^{2k}+y^{2k} \le (x^2 + y^2)^k = (\lambda^2\cos^2\psi + \lambda^2 \sin^2 \psi)^k=\lambda^{2k}$ and $2 \cdot (2k)! \ge 4^k \cdot k!$ for all integers $k \ge 1$. The last fact can be easily verified by induction: For $k=1$ both sides are equal to $4$. Moving from $k$ to $k+1$ multiplies the right side by $4(k+1)$ and the left side by $(2k+1)(2k+2)=4(k+1)(k+1/2) > 4(k+1)$.
    \end{proof}
    \begin{lem}\label{lem:theta-cts}
        Let $\theta \in [0,2\pi)$ be uniformly random. Let $\lambda, \psi \in \mathbb{R}$ be fixed. Then $\ex{\theta}{\exp(\lambda\cdot\cos(\theta+\psi)}\le\exp(\lambda^2/4)$.
    \end{lem}
    \begin{proof}
        Since $\theta$ is uniform on $[0,2\pi)$ and $\cos$ is a periodic function with period $2\pi$, the distribution of $\cos(\theta+\psi)$ is the same as that of $\cos \theta$, so we may ignore $\psi$. 
        Also the distribution is symmetric (i.e., the distribution of $-\cos \theta = \cos(\theta+\pi)$ is the same as that of $\cos \theta$).
        Thus $\ex{}{\cos^k \theta}=0$ for all odd $k$. 
        We also have $\ex{}{\cos^2 \theta} = \ex{}{\frac{1+\cos(2\theta)}{2}} = \frac12$.
        Integration by parts yields the recurrence $\ex{\theta}{\cos^k \theta} = \frac{k-1}{k} \ex{\theta}{\cos^{k-2} \theta}$ for $k \ge 2$. This yields $\ex{\theta}{\cos^{2k} \theta} = \frac{(2k-1)!}{2^{2k-1} \cdot k! \cdot (k-1)!}$ for all integers $k \ge 1$.
        Thus
        \begin{align*}
            \ex{\theta}{\exp(\lambda \cdot \cos \theta)} &= 1 + \sum_{k=1}^\infty \frac{\lambda^{2k}}{(2k)!} \ex{\theta}{\cos^{2k} \theta}\\
            &= 1 + \sum_{k=1}^\infty \frac{\lambda^{2k} \cdot (2k-1)!}{(2k)! \cdot 2^{2k-1} \cdot k! \cdot (k-1)!}\\
            &= 1 + \sum_{k=1}^\infty \frac{(\lambda^2/4)^k}{k!} \cdot \frac{1}{k!}\\
            &\le  1 + \sum_{k=1}^\infty \frac{(\lambda^2/4)^k}{k!}\\
            &= \exp(\lambda^2/4).
        \end{align*}
    \end{proof}
    
    We emphasize that the discrete fourier transform (i.e., matrix-vector multiplications with $H_{2d}''$ from Equation \ref{eqn:Hfourier}) can be computed in $O(d \log d)$ operations for \emph{any} $d$ -- not just powers of $2$. Although the exact efficiency (i.e., constants) depends on $d$ \cite{Wiki-FFT}.
\end{itemize}
\subsection{Modular Clipping}

In this section we cover third and final source of error -- modular arithmetic. This is introduced by the secure aggregation procedure.

We first define the modular clipping operation in a convenient form for real numbers.
\newcommand{\modular}[1]{M_{[#1]}}
\begin{defn}\label{defn:modular}
For $a<b$, define $\modular{a,b} : \mathbb{R} \to [a,b]$ by $M(x) = x + (b-a) \cdot n$ where $n \in \mathbb{Z}$ is chosen so that $x + (b-a) \cdot n \in [a,b]$. (Ties are broken arbitrarily.) We also define $\modular{a,b}(x) = (\modular{a,b}(x_1), \modular{a,b}(x_2), \cdots, \modular{a,b}(x_d)) \in [a,b]^d$ for $x \in \mathbb{R}^d$.
\end{defn}

Modular arithmetic is usually performed over $\mathbb{Z}_m$, which we equate with the set of integers $\{0,1,2,\cdots,m-1\}$. Our algorithm maps real values to this set of integers. However, for our analysis, it will be convenient to imagine the modular clipping operation taking place directly over the real numbers. (This is completely equivalent to the usual view, but has the advantage of allowing us to perform the analysis over a centered interval $[-r,r]$, rather than $[0,m]$.) Specifically, our algorithm can be thought of as performing modular arithmetic over $\{\gamma(1-m/2), \gamma(2-m/2), \cdots, -\gamma, 0, \gamma, \cdots, \gamma(m/2-1), \gamma(m/2)\}$ instead of $\mathbb{Z}_m$. This operation is denoted as $\modular{-m\gamma/2,m\gamma/2}$.

Note that definition \ref{defn:modular} does not specify whether $\modular{a,b}(a)=a$ or $\modular{a,b}(a)=b$ (and likewise for $\modular{a,b}(b)$). Thus our analysis does not depend on how this choice is made. 

A key property of the modular operation is that it is homomorphic:
\begin{equation}
    \forall a<b ~~\forall x, y \in \mathbb{R} ~~~~~ \modular{a,b}(x+y) = \modular{a,b}(\modular{a,b}(x)+\modular{a,b}(y)).
\end{equation}

Our goal is to analyze $\modular{a,b}(A(x))$, where $A$ is as in Proposition \ref{prop:dg-rr-sum}. The modular clipping arises from the secure aggregation step, which works over a finite group. Note that $A$ discretizes the values (although this is not crucial for this part of the analysis).

We want to ensure that $\modular{a,b}(A(x)) \approx x$. We have already established that $A(x) \approx x$ and our goal is now to analyze the modular clipping operation. If $A(x) \in [a,b]^d$, then $\modular{a,b}(A(x))=A(x)$ and we are in good shape; thus our analysis centers on ensuring that this is the case.

We will use the fact that the flattening operation, as well as the randomized rounding and noise addition, result in each coordinate being a centered subgaussian random variable. This allows us to bound the probability of straying outside $[a,b]$.

First we present a technical lemma.

\begin{lem}\label{lem:mod-exp}
    Let $r>0$ and $x \in \mathbb{R}$. Then \[|\modular{-r,r}(x)-x| \le 2r \cdot \left( \exp \left( \frac12 t \cdot\left( \frac{x}{r}-1\right) \right) + \exp \left( \frac12 t \cdot \left(\frac{-x}{r} - 1 \right) \right) \right)\] and \[(\modular{-r,r}(x)-x)^2 \le 4r^2 \cdot \left( \exp \left( t \cdot\left( \frac{x}{r}-1\right) \right) + \exp \left( t \cdot \left(\frac{-x}{r} - 1 \right) \right) \right)\] for all $t \ge \log 2$.
\end{lem}

\ifarxiv 
\begin{proof}
    We have $|\modular{-r,r}(x)-x| \le 2r \cdot \left\lfloor \frac{|x|+r}{2r} \right\rfloor$ for all $x \in \mathbb{R}$. 
    We also have $\lfloor x \rfloor \le e^{t\cdot(x-1)}$ for all $x \in \mathbb{R}$ and all $t \ge \log 2$. (For $x<1$, we have $\lfloor x \rfloor \le 0 \le e^{t \cdot (x-1)}$. For $1 \le x < 2$, we have $\lfloor x \rfloor = e^0 \le e^{t(x-1)}$. We have $e^{t(1-1)}=1$ and $e^{t(2-1)} \ge 2$; since $e^{t(x-1)}$ is convex, this implies $\lfloor x \rfloor \le x \le e^{t(x-1)}$ for $x\ge 2$.)
    Thus \[|\modular{-r,r}(x)-x| \le 2r \cdot \exp \left( t \cdot \frac{|x|-r}{2r} \right) \le 2r \cdot \left( \exp \left( t \cdot \frac{x-r}{2r} \right) + \exp \left( t \cdot \frac{-x-r}{2r} \right) \right)\] and, hence, 
    \begin{align*}
        (\modular{-r,r}(x)-x)^2 &\le 4r^2 \cdot \exp \left( t \cdot \frac{|x|-r}{r} \right)
        \le 4r^2 \cdot \exp \left( t \cdot \frac{x-r}{r} \right) + 4r^2 \cdot \exp \left( t \cdot \frac{-x-r}{r} \right)
    \end{align*}
    for all $t \ge \log 2$ and $x \in \mathbb{R}$, as required.
\end{proof}
\fi

Now we have our bound on the error of modular clipping:

\begin{prop}[Error of Modular Clipping]\label{prop:mod-err}
    Let $a<b$ and $\omega,\sigma>0$ satisfy $\sigma \le (b-a)/2$.
    Let $X \in \mathbb{R}$ satisfy $\ex{}{e^{tX}} \le \omega \cdot e^{t^2\sigma^2/2}$ for all $t \in \mathbb{R}$.
    Then \[\ex{}{\left| \modular{a,b}(X)-X \right|} \le (b-a) \cdot \omega \cdot e^{-(b-a)^2/8\sigma^2}\cdot \left( e^{\frac{a^2-b^2}{4\sigma^2}} + e^{\frac{b^2-a^2}{4\sigma^2}} \right)\] and \[\ex{}{\left( \modular{a,b}(X)-X \right)^2} \le (b-a)^2 \cdot \omega \cdot e^{-(b-a)^2/8\sigma^2}\cdot \left( e^{\frac{a^2-b^2}{4\sigma^2}} + e^{\frac{b^2-a^2}{4\sigma^2}} \right).\]
\end{prop}

\ifarxiv 
\begin{proof}
    First we center: Let $c=(a+b)/2$ and $r=(b-a)/2$. Then $\modular{a,b}(x) = \modular{-r,r}(x-c)+c$ and  $|\modular{a,b}(x)-x|=|\modular{-r,r}(x-c)-(x-c)|$ for all $x \in \mathbb{R}$. Let $X' = X-c$.
    
    By Lemma \ref{lem:mod-exp}, for all $t \ge \log 2$,
    \begin{align*}
        \ex{}{\left| \modular{a,b}(X)-X \right|} &= \ex{}{\left| \modular{-r,r}(X')-X' \right|} \\
        &\le 2r \cdot \ex{}{ \exp \left( \frac12 t \cdot\left( \frac{X'}{r}-1\right) \right) + \exp \left( \frac12 t \cdot \left(\frac{-X'}{r} - 1 \right) \right) }\\
        &\le 2r \cdot \omega \cdot e^{t^2\sigma^2/8r^2-t/2} \cdot \left( e^{-tc/2r} + e^{tc/2r} \right)\\
        &= (b-a) \cdot \omega \cdot e^{t^2\sigma^2/2(b-a)^2-t/2} \cdot \left( e^{-\frac{t}{2}\cdot\frac{a+b}{b-a}} + e^{\frac{t}{2}\cdot\frac{a+b}{b-a}} \right).
    \end{align*}
    Set $t=(b-a)^2/2\sigma^2 \ge \log 2$ to obtain the first part of the result.
    
    By Lemma \ref{lem:mod-exp}, for all $t \ge \log 2$,
    \begin{align*}
        \ex{}{\left( \modular{a,b}(X)-X \right)^2} &= \ex{}{\left( \modular{-r,r}(X')-X' \right)^2} \\
        &\le 4r^2 \cdot \ex{}{ \exp \left( t \cdot\left( \frac{X'}{r}-1\right) \right) + \exp \left( t \cdot \left(\frac{-X'}{r} - 1 \right) \right) }\\
        &\le 4r^2 \cdot \omega \cdot e^{t^2\sigma^2/2r^2-t} \cdot \left( e^{-tc/r} + e^{tc/r} \right)\\
        &= (b-a)^2 \cdot \omega \cdot e^{2t^2\sigma^2/(b-a)^2-t} \cdot \left( e^{-t\frac{a+b}{b-a}} + e^{t\frac{a+b}{b-a}} \right).
    \end{align*}
    Set $t=(b-a)^2/4\sigma^2 \ge \log 2$ to obtain the second part of the result.
\end{proof}
\fi

\subsection{Putting Everything Together}

We have now analyzed the three sources of error -- randomized rounding, privacy-preserving noise, and modular arithmetic. It remains to combine these results. This yields our main result:

\begin{thm}[Main Theoretical Result]\label{thm:main-guarantee}
Let $\beta \in [0,1)$, $\sigma^2 \ge \frac12 \gamma > 0$, and $c>0$.
Let $n, d \in \mathbb{N}$ and $\rho \ge 1$. Let $U \in \mathbb{R}^{d \times d}$ be a random unitary matrix such that \[\forall x \in \mathbb{R}^d ~~ \forall i \in [d] ~~ \forall t \in \mathbb{R} ~~~~~ \ex{}{\exp(t (Ux)_i)} \le \exp(t^2 \rho \|x\|_2^2 / 2d).\]
Let 
\begin{align}
    \Delta_2^2 &:= \min \left\{ \begin{array}{c} c^2 + \frac14 \gamma^2 d + \sqrt{2\log(1/\beta)} \cdot \gamma \cdot \left(c + \frac12 \gamma \sqrt{d}\right) ,\\ \left( c + \gamma \sqrt{d} \right)^2 \end{array} \right\},\\
    G &:= \left\{ y \in \mathbb{R}^d : \|y\|_2^2 \le \Delta_2^2 \right\},\\
    \tau &:= 10 \cdot \sum_{k=1}^{n-1} e^{-2\pi^2\frac{\sigma^2}{\gamma^2} \cdot \frac{k}{k+1}},\\
    \varepsilon &:= \min \left\{ \begin{array}{c} \sqrt{ \frac{\Delta_2^2}{n\sigma^2} + 2 \tau d} ,\\ \frac{\Delta_2}{\sqrt{n} \sigma} + \tau \sqrt{d} \end{array} \right\}.
\end{align} 
Let $R_\gamma^G$ be as in Definition \ref{defn:crr}.
Let $r>0$ and let $\modular{-r,r}$ be as in Definition \ref{defn:modular}.
Define a randomized algorithm $\tilde A : (\mathbb{R}^d)^n \to \gamma \mathbb{Z}^d$ by
\begin{equation}
    \tilde A(x) = U^T \modular{-r,r}\left( \sum_i^n R_\gamma^G\left(\min\left\{1, \frac{c}{\|x_i\|_2} \right\} \cdot U x_i\right) + \gamma \cdot Y_i \right),
\end{equation}
where $Y_1, \cdots, Y_n \in \mathbb{Z}^d$ are independent random vectors with each entry drawn independently from $\mathcal{N}_{\mathbb{Z}}(0,\sigma^2/\gamma^2)$.

Then $\tilde A$ satisfies $\frac12 \varepsilon^2$-concentrated differential privacy.\footnote{Note that this is with respect to the addition or removal of an element, not replacement. To keep $n$ fixed, we would need addition/removal to be defined to simply zero-out the relevant vectors.} 

Let $x_1, \cdots, x_n \in \mathbb{R}^d$ with $ \|x_i\|_2 \le c$ for all $i \in [n]$. Let 
\begin{equation}
    \hat\sigma^2(x) := \frac{\rho}{d} \left\| \sum_i^n x_i \right\|_2^2 + \left( \frac{\gamma^2}{4} + \sigma^2 \right) \cdot n \le \frac{\rho}{d} c^2 n^2 + \left( \frac{\gamma^2}{4} + \sigma^2 \right) \cdot n
\end{equation}
If $\hat\sigma^2(x) \le r^2$, then 
\begin{equation}
    \ex{}{\left\|\tilde{A}(x) - \sum_i^n x_i\right\|_2^2} \le \frac{d \cdot n}{1-\beta} \cdot \left( \frac{2\sqrt{2} \cdot r \cdot e^{-r^2/4\hat\sigma^2(x)}}{\sqrt{n\cdot (1-\beta)^{n-1}}} + \sqrt{\frac{\gamma^2}{4} + \frac{\beta^2\gamma^2 n}{1-\beta} + (1-\beta)\sigma^2}\right)^2.\label{eqn:main-guarantee}
\end{equation}
\end{thm}

There is a lot to unpack in Theorem \ref{thm:main-guarantee}. Let us work through the parameters:
\begin{itemize}
    \item $n$ is the number of individuals and $d$ is the dimension of the data.
    \item $c$ is the bound on $2$-norm of the individual data vectors.
    \item $\gamma$ is the granularity of the discretization -- we round data vectors to the integer grid $\gamma \mathbb{Z}^d$.
    \item $r$ is the range for the modular clipping -- our final sum ends up being clipped to $\gamma \mathbb{Z}^d \cap [-r,r]^d$. The secure aggregation does modular arithmetic over a group of size $m = 2r/\gamma$ (note $r$ should be a multiple of $\gamma$). This ratio determines the communication complexity.
    \item $\sigma^2$ is the variance of the individual discrete Gaussian noise that we add; the sum will have variance $n\sigma^2$. This determines the privacy; specifically $\varepsilon \approx \frac{c}{\sqrt{n}\sigma}$ and we attain $\frac12\varepsilon^2$-concentrated differential privacy.
    \item $\beta$ is a parameter that controls the conditional randomized rounding. $\beta=0$ yields unconditional randomized rounding, and larger $\beta$ entails more aggressive conditioning. It will be helpful to think of $\beta = \sqrt{\gamma/n}$; although, in practice, slightly larger $\beta$ may be preferable.
    \item $\rho$ measures how good the flattening matrix $U$ is (cf.~Lemma \ref{lem:flatHD}). Think of $\rho=1$ or at most $\rho\le2$.
    \item The other parameters -- $\Delta_2$, $G$, $\tau$, $\hat\sigma$ -- are not important, as they are determined by the previous parameters. $\Delta_2$ and $G$ determine how much the conditional randomized rounding can increase the norm (initially the norm is $c$). $\tau$ quantifies how far the sum of discrete Gaussians is from just a single discrete Gaussian and how this affects the differential privacy guarantee. The ratio $\hat\sigma/r$ measures how much error the modular clipping contributes. $\hat\sigma$ is determined by other parameters, but note that $\left\| \sum_i^n x_i \right\| \le \sum_i^n \|x_i\| \le cn$ may be a loose upper bound, in which case, the clipping error is less.
\end{itemize}
Now we look at the error bound. If we assume $\beta \le 1/\sqrt{n}$ and $\hat\sigma^2(x) \le r^2/4\log(r\sqrt{n}/\gamma^2)$, then the guarantee \eqref{eqn:main-guarantee} is simply \[\ex{}{\left\|\tilde{A}(x)-\sum_i^n x_i\right\|_2^2} \le O\left(dn\left(\sigma^2 + \gamma^2\right)\right). \]
The first term is roughly the cost of privacy -- $dn\sigma^2 \approx d \frac{c^2}{\varepsilon^2}$. The second term, $dn\gamma^2$, is the cost of randomized rounding and modular clipping. (We have assumed $\beta$ and $\hat\sigma$ are sufficiently small to avoid any additional terms.) 

\ifarxiv 
\begin{proof}[Proof of Theorem \ref{thm:main-guarantee}]
    The differential privacy guarantee follows from the postprocessing property of differential privacy and Proposition \ref{prop:dg-rr-sum} (which, in turn, applies Proposition \ref{prop:dg-sum-priv}).
    
    Now we turn to the utility analysis. Let $A$ be as in Proposition \ref{prop:dg-rr-sum}. Then $\tilde{A}(x) = U^T \modular{-r,r}(A(Ux))$, where $Ux = (Ux_1, Ux_2, \cdots, Ux_n)$.
    Proposition \ref{prop:dg-rr-sum} gives us the following guarantees.
    \begin{align*}
    \left\|\ex{}{A(Ux)}-U\sum_i^n x_i\right\|_2 &\le \frac{\beta \cdot \gamma \cdot \sqrt{d} \cdot n}{1-\beta}.\\
    \ex{}{\left\| A(Ux) - \ex{}{A(Ux)} \right\|_2^2 } &\le \frac{\gamma^2 \cdot d \cdot n}{4(1-\beta)} + n \cdot d \cdot \sigma^2.\\
    \ex{}{\left\| A(Ux) - U\sum_i^n x_i \right\|_2^2 } &\le \frac{\gamma^2 \cdot d \cdot n}{4(1-\beta)} + \left( \frac{\beta}{1-\beta} \gamma \sqrt{d} n \right)^2 +  n \cdot d \cdot \sigma^2\\
    \forall t \in \mathbb{R} ~ \forall j \in [d] ~~~ \ex{}{\exp\left( t \cdot \left( A(Ux) - U\sum_i^n x_i \right)_j \right)} &\le \frac{\exp \left( \left(\frac{\gamma^2}{8} + \frac{\sigma^2}{2} \right) \cdot t^2 \cdot n \right)}{(1-\beta)^n}.
    \end{align*}
    By our assumption on $U$ (and independence) we have \[\forall t \in \mathbb{R} ~ \forall j \in [d] ~~~ \ex{}{\exp\left( t \cdot \left( A(Ux) \right)_j \right)} \le \exp\left(\frac{t^2\rho}{2d} \left\| \sum_i^n x_i \right\|_2^2\right) \cdot \frac{\exp \left( \left(\frac{\gamma^2}{8} + \frac{\sigma^2}{2} \right) \cdot t^2 \cdot n \right)}{(1-\beta)^n}.\]
    Recall $\hat\sigma^2(x) = \frac{\rho}{d} \left\| \sum_i^n x_i \right\|_2^2 + \left( \frac{\gamma^2}{4} + \sigma^2 \right) \cdot n$.
    By Proposition \ref{prop:mod-err}, for all $j \in [d]$, \[\ex{}{\left| \modular{a,b}(A(Ux))_j-A(Ux)_j \right|} \le (b-a) \cdot \frac{1}{(1-\beta)^n} \cdot e^{-(b-a)^2/8\hat\sigma^2(x)}\cdot \left( e^{\frac{a^2-b^2}{4\hat\sigma^2}} + e^{\frac{b^2-a^2}{4\hat\sigma^2}} \right)\] and \[\ex{}{\left( \modular{a,b}(A(Ux))_j-A(Ux)_j \right)^2} \le (b-a)^2 \cdot \frac{1}{(1-\beta)^n} \cdot e^{-(b-a)^2/8\hat\sigma^2(x)}\cdot \left( e^{\frac{a^2-b^2}{4\hat\sigma^2}} + e^{\frac{b^2-a^2}{4\hat\sigma^2}} \right),\] where $a=-r$ and $b=r$ here.
    Summing over $j \in [d]$ yields 
    \begin{align*}
        \ex{}{\left\| \modular{-r,r}(A(Ux))-A(Ux) \right\|_2^2} \le 4r^2 \cdot \frac{d}{(1-\beta)^n} \cdot e^{-r^2/2\hat\sigma^2(x)}\cdot 2.
    \end{align*}
    For all $u,v \in \mathbb{R}$, we have $(u+v)^2 = \inf_{\lambda>0} (1+\lambda)u^2+(1+1/\lambda)v^2$. Thus
    \begin{align*}
        &\ex{}{\left\|\tilde{A}(x) - \sum_i^n x_i\right\|_2^2} \\
        &= \ex{}{\left\|\left(U^T \modular{-r,r}(A(Ux)) - U^T A(Ux)\right) + \left( U^T A(Ux) - U^T \sum_i^n Ux_i \right)\right\|_2^2}\\
        &= \ex{}{\left\|\left(\modular{-r,r}(A(Ux)) - A(Ux)\right) + \left( A(Ux) - \sum_i^n Ux_i \right)\right\|_2^2}\\
        &\le \inf_{\lambda>0} (1+\lambda)\ex{}{\left\|\modular{-r,r}(A(Ux)) - A(Ux)\right\|_2^2} + (1+1/\lambda)\ex{}{\left\| A(Ux) - \sum_i^n Ux_i \right\|_2^2}\\
        &\le \inf_{\lambda>0} (1+\lambda) \cdot 4r^2 \cdot \frac{d}{(1-\beta)^n} \cdot e^{-r^2/2\hat\sigma^2(x)}\cdot 2 + (1+1/\lambda) \cdot \left( \frac{\gamma^2 \cdot d \cdot n}{4(1-\beta)} + \left( \frac{\beta}{1-\beta} \gamma \sqrt{d} n \right)^2 +  n \cdot d \cdot \sigma^2 \right) \\
        &= \left(\sqrt{8r^2 \cdot \frac{d}{(1-\beta)^n} \cdot e^{-r^2/2\hat\sigma^2(x)}} + \sqrt{\frac{\gamma^2 \cdot d \cdot n}{4(1-\beta)} + \left( \frac{\beta}{1-\beta} \gamma \sqrt{d} n \right)^2 +  n \cdot d \cdot \sigma^2} \right)^2\\
        &= \frac{d \cdot n}{1-\beta} \cdot \left( \frac{2\sqrt{2} \cdot r \cdot e^{-r^2/4\hat\sigma^2(x)}}{\sqrt{n (1-\beta)^{n-1}}} + \sqrt{\frac{\gamma^2}{4} + \frac{\beta^2\gamma^2 n}{1-\beta} + (1-\beta)\sigma^2}\right)^2.
    \end{align*}
\end{proof}
\fi

We gave an asymptotic version of Theorem \ref{thm:main-guarantee} in the introduction.

\ifarxiv 
Finally, we analyse how to set the parameters to obtain this bound. Note that we do not attempt to optimize constants here at all.
\begin{proof}[Proof of Theorem \ref{thm:intro-accuracy}.]
Theorem \ref{thm:main-guarantee} gives the following parameters.
\begin{align*}
    \Delta_2^2 &:= \min \left\{ \begin{array}{c} c^2 + \frac14 \gamma^2 d + \sqrt{2\log(1/\beta)} \cdot \gamma \cdot \left(c + \frac12 \gamma \sqrt{d}\right) ,\\ \left( c + \gamma \sqrt{d} \right)^2 \end{array} \right\},\\
    \tau &:= 10 \cdot \sum_{k=1}^{n-1} e^{-2\pi^2\frac{\sigma^2}{\gamma^2} \cdot \frac{k}{k+1}},\\
    \varepsilon &:= \min \left\{ \begin{array}{c} \sqrt{ \frac{\Delta_2^2}{n\sigma^2} + 2 \tau d} ,\\ \frac{\Delta_2}{\sqrt{n} \sigma} + \tau \sqrt{d} \end{array} \right\},\\
    \hat\sigma^2(x) &:= \frac{\rho}{d} \left\| \sum_i^n x_i \right\|_2^2 + \left( \frac{\gamma^2}{4} + \sigma^2 \right) \cdot n \\
    &\le \frac{\rho}{d} c^2 n^2 + \left( \frac{\gamma^2}{4} + \sigma^2 \right) \cdot n,\\
    \ex{}{\left\|\tilde{A}(x) - \sum_i^n x_i\right\|_2^2} &\le \frac{d \cdot n}{1-\beta} \cdot \left( \frac{2\sqrt{2} \cdot r \cdot e^{-r^2/4\hat\sigma^2(x)}}{\sqrt{n (1-\beta)^{n-1}}} + \sqrt{\frac{\gamma^2}{4} + \frac{\beta^2\gamma^32n}{1-\beta} + (1-\beta)\sigma^2}\right)^2.
\end{align*}
Note that $r=\frac12 \gamma m$.
All we must do is verify that setting the parameters as specified in Theorem \ref{thm:intro-accuracy} yields $\frac12\varepsilon^2$-concentrated DP and the desired accuracy.
First,
\begin{align*}
    \varepsilon^2 &\le \frac{\Delta_2^2}{n\sigma^2} + 2 \tau d\\
    &\le \frac{(c+\gamma\sqrt{d})^2}{n\sigma^2} + 20 nd e^{-\pi^2 (\sigma/\gamma)^2}\\
    &\le  \frac{2c^2}{n \sigma^2} + \frac{2d}{n(\sigma/\gamma)^2} + 20nd e^{-\pi^2 (\sigma/\gamma)^2}.
\end{align*}
Thus the privacy requirement is satisfied as long as $\sigma \ge 2c/\varepsilon\sqrt{n}$ and $(\sigma/\gamma)^2 \ge 8d/\varepsilon^2n$,  and $20nd e^{\pi^2 (\sigma/\gamma)^2} \le \varepsilon^2/4$. So we can set \[ \sigma = \max \left\{ \frac{2c}{\varepsilon \sqrt{n}} , \frac{\gamma \sqrt{8d}}{\varepsilon \sqrt{n}}, \frac{\gamma}{\pi^2} \log \left( \frac{80nd}{\varepsilon^2} \right) \right\} = \tilde{\Theta}\left(\frac{c}{\varepsilon\sqrt{n}}+\sqrt{\frac{d}{n}}\cdot\frac{\gamma}{\varepsilon}\right).\]

We set $\beta = \min\{ 1/n , 1/2 \} = \Theta\left(\frac 1 n\right)$.

Next
\begin{align*}
    \hat\sigma^2 &\le \frac{\rho}{d} c^2 n^2 + \left( \frac{\gamma^2}{4} + \sigma^2 \right) \cdot n,\\
    &\le \frac{c^2n^2}{d} + \gamma^2 n + \sigma^2 n\\
    &\le O\left( \frac{c^2n^2}{d} + \gamma^2 n + \frac{c^2}{\varepsilon^2} + \frac{\gamma^2 d}{\varepsilon^2} + \gamma^2 n \log^2 \left( \frac{nd}{\varepsilon^2} \right)\right)\\
    &\le O\left( \frac{c^2n^2}{d}  + \frac{c^2}{\varepsilon^2} \right) + \gamma^2 \cdot O\left(n + \frac{d}{\varepsilon^2} +  n \log^2 \left( \frac{nd}{\varepsilon^2} \right)\right).
\end{align*}

Now we work out the asymptotics of the accuracy guarantee:
\begin{align*}
    &\!\ex{}{\left\|\tilde{A}(x) - \sum_i^n x_i\right\|_2^2} \\
    &\le \frac{d \cdot n}{1-\beta} \cdot \left( \frac{2\sqrt{2} \cdot r \cdot e^{-r^2/4\hat\sigma^2(x)}}{\sqrt{n(1-\beta)^{n-1}}} + \sqrt{\frac{\gamma^2}{4} + \frac{\beta^2\gamma^2 n}{1-\beta} + (1-\beta)\sigma^2}\right)^2\\
    &\le O \left( nd \left( \frac{r e^{-r^2/4\hat\sigma^2}}{\sqrt{n}} + \sqrt{\gamma^2 + \gamma^2 + \sigma^2}\right)^2 \right)\\
    &\le O \left( nd \left( \frac{r^2}{n} e^{-r^2/2\hat\sigma^2} + \gamma^2 + \sigma^2 \right)\right)\\
    &\le O \left( nd \left( \frac{\gamma^2m^2}{n} \exp\left(\frac{-\gamma^2m^2}{8\hat\sigma^2}\right) + \gamma^2 + \frac{c^2}{\varepsilon^2 n} + \frac{\gamma^2 d}{\varepsilon^2 n} + \gamma^2 \log^2 \left( \frac{nd}{\varepsilon^2} \right) \right)\right)\\
    &= O \left(\frac{c^2 d}{\varepsilon^2} + \gamma^2 nd \left( \frac{m^2}{n}\exp\left(\frac{-\gamma^2m^2}{8\hat\sigma^2}\right) + 1 +    \frac{d}{\varepsilon^2 n} + \log^2 \left( \frac{nd}{\varepsilon^2} \right) \right) \right).
\end{align*}
Now we wish to set $\gamma$ so that $ \frac{m^2}{n}\exp\left(\frac{-\gamma^2m^2}{8\hat\sigma^2}\right) \le 1$ -- i.e., $\gamma \ge \frac{\hat\sigma}{m} \sqrt{8 \log(1+m^2/n)}$. However, simply setting $\gamma = \frac{\hat\sigma^*}{m} \sqrt{8 \log(1+m^2/n)}$, where $\hat\sigma^*$ is our upper bound on $\hat\sigma$, is cyclic, because our bound on $\hat\sigma$ depends on $\gamma$. 
Fortunately, we can resolve this as long as the coefficient in this cycle is $\le 1 - \Omega(1)$. That coefficient is $O\left(n + \frac{d}{\varepsilon^2} +  n \log^2 \left( \frac{nd}{\varepsilon^2} \right)\right) \cdot \frac{\log(1+m^2/n)}{m^2}$.
A sufficient condition for this is \[m^2 \ge O\left(n + \frac{d}{\varepsilon^2} +  n \log^2 \left( \frac{nd}{\varepsilon^2} \right)\right) \cdot {\log(1+m^2/n)} = \tilde{O}\left(n + d/\varepsilon^2\right).\]
Thus we can set \[\gamma^2 = O\left( \frac{c^2n^2}{d}  + \frac{c^2}{\varepsilon^2} \right) \cdot \frac{\log(1+m^2/n)}{m^2} \] and we obtain
\begin{align*}
    &\!\ex{}{\left\|\tilde{A}(x) - \sum_i^n x_i\right\|_2^2} \\
    &\le  O \left(\frac{c^2 d}{\varepsilon^2} + \gamma^2 nd \left( 1  + 1 +    \frac{d}{\varepsilon^2n} + \log^2 \left( \frac{nd}{\varepsilon^2} \right) \right) \right)\\
    &\le  O \left(\frac{c^2 d}{\varepsilon^2} + \left( \frac{c^2n^2}{d}  + \frac{c^2}{\varepsilon^2} \right) \cdot \frac{\log(1+m^2/n)}{m^2} nd \left( 1  +   \frac{d}{\varepsilon^2n} + \log^2 \left( \frac{nd}{\varepsilon^2} \right) \right) \right)\\
    &\le  O \left(\frac{c^2 d}{\varepsilon^2} + \frac{c^2 d}{\varepsilon^2} \left( \frac{\varepsilon^2 n^2}{d}  + 1 \right) \cdot \frac{\log(1+m^2/n)}{m^2} \left( n  +   \frac{d}{\varepsilon^2} + n \log^2 \left( \frac{nd}{\varepsilon^2} \right) \right) \right)\\
    &\le  O \left(\frac{c^2 d}{\varepsilon^2} \left( 1 + \frac{\log(1+m^2/n)}{m^2} \cdot  \left( n^2 + \frac{d}{\varepsilon^2} + \left( \frac{\varepsilon^2 n^3}{d} + n \right)\log^2 \left( \frac{nd}{\varepsilon^2} \right) \right) \right)\right).
\end{align*}
Thus, if \[m^2 \ge O\left({\log(1+m^2/n)} \cdot  \left( n^2 + \frac{d}{\varepsilon^2} + \left( \frac{\varepsilon^2 n^3}{d} + n \right)\log^2 \left( \frac{nd}{\varepsilon^2} \right) \right)\right) = \tilde{O}\left(n^2 + \frac{\varepsilon^2n^3}{d} + \frac{d}{\varepsilon^2} \right),\] then the mean squared error is $O(c^2d/\varepsilon^2)$, as required.
\end{proof}
\fi

\section{Experiments} \label{sec:experiments}
\begin{figure*}[t]
    \centering
    \includegraphics[width=\linewidth]{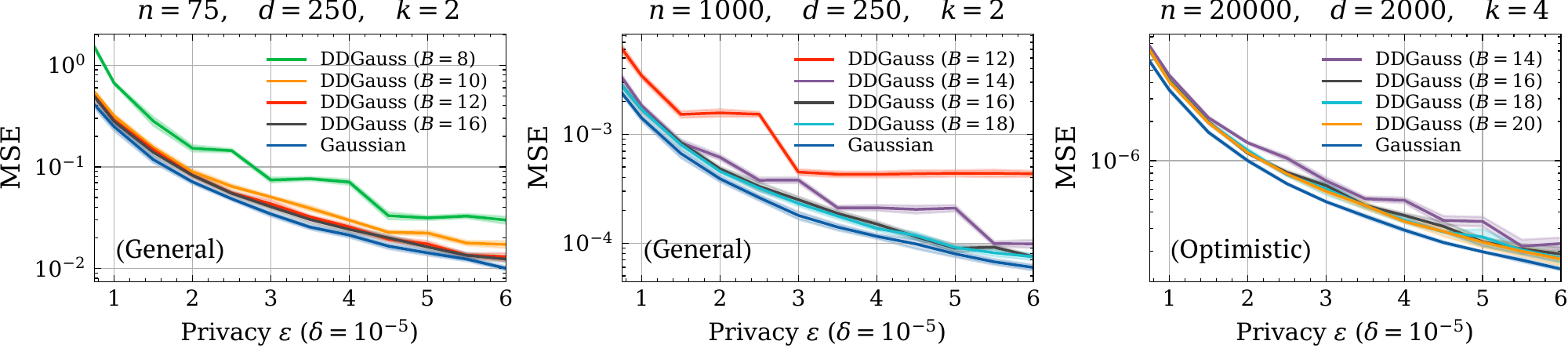}
    \vspace{-2em}
    \caption{Distributed mean estimation with the distributed discrete Gaussian. $n$: number of clients. $d$: vector dimension. $k$: number of stddevs of $\sum_i^n  \tilde{x}_i + y_i$ to bound. $B$: per-coordinate bit-width. General/Optimistic: assumes $\left\| \sum_i^n x_i \right\|_2 \le cn$ or $\le c\sqrt n$ for choosing $\gamma$.}
    \label{fig:main-dme-results}
\end{figure*}

We empirically evaluate the distributed discrete Gaussian mechanism (DDGauss) on two tasks: distributed mean estimation (DME) and federated learning (FL). 
Our goal is to demonstrate that the utility of DDGauss matches that of the continuous Gaussian mechanism under the same privacy guarantees when given sufficient communication budget.
For both tasks, the top-level parameters include the number of participating clients $n$, the $\ell_2$ norm bound for the client vectors $c$, the dimension $d$, the privacy budget $\varepsilon$, and the bit-width $B$ which determines the modulo field size $m = 2^B$. For FL, we also consider the number of rounds $T$ and the total number of clients $N$ from which we randomly sample $n$ clients in each round. 
We fix the conditional rounding bias to $\beta = e^{-1/2}$ unless otherwise stated.

To select the granularity parameter $\gamma$, we carefully balance the errors from randomized rounding and modular clipping.
From the earlier sections, we know that each entry of $\sum_i^n \tilde{x}_i + y_i$ is subgaussian with known constants. Thus, for a fixed $B$, we can choose $\gamma$ to ensure that the modular clipping range includes $k$ standard deviations of $\sum_i^n \tilde{x}_i + y_i$.
Specifically, the heuristic is to select $\gamma$ such that $2k\hat\sigma$ is bounded within the modulo field size $2^B$ where $\hat\sigma^2 =  \frac{c^2 n^2} {d} + \left( \frac{\gamma^2}{4} + \sigma^2 \right) \cdot n$ (Theorem~\ref{thm:main-guarantee}).
Here, $k$ captures the trade-off between the errors from quantization and modular clipping and should be application-dependent. A small $k$ leads to a small $\gamma$ and thus less error from rounding but more error from modular clipping; a large $k$ means modular clipping happens rarely but at a cost of more rounding error.

\subsection{Distributed Mean Estimation} 

In this experiment, $n$ clients each hold a $d$-dimensional vector $x_i$ uniformly randomly sampled from the $\ell_2$ sphere with radius $c=10$. 
We compute the ground truth mean vector $\bar x = \frac{1}{n} \sum_i^n x_i$ as well as the differentially private mean estimates $\widehat x$ across different mechanisms and communication/privacy budgets. We use the analytic Gaussian mechanism~\cite{balle2018improving} as the strong baseline.
Figure~\ref{fig:main-dme-results} shows the mean MSE $\|\bar x - \widehat x\|^2_2 / d$ with 95\% confidence interval over 10 random dataset initializations.\footnote{The kinks on the low bit-width curves are due to the TensorFlow implementation of the discrete Gaussian sampler taking integer noise scales; to preserve privacy, noise scales are rounded up as $\lceil\sigma / \gamma\rceil$ in all experiments.}
The first two plots assume a general norm bound $\left\| \sum_i^n x_i \right\|_2 \le cn$ when choosing $\gamma$ (generally applicable to FL applications), while the third plot assumes an optimistic bound $\left\| \sum_i^n x_i \right\|_2 \le c\sqrt n$ as $x_i$'s are sampled uniformly randomly on the $\ell_2$ sphere. 
Note that the bit-width $B$ applies to each coordinate of the quantized and noisy aggregate.
Results indicate that DDGauss achieves a good communication-utility trade-off and matches the Gaussian with sufficient bit-widths. 

In Figure~\ref{fig:dme-results-extended}, we additionally investigate the trade-off between quantization errors and modular clipping errors by trying different values of $k$. Here, we use the optimistic norm bound on the vector sum as the general norm bound could be loose (thus $\gamma$ would be chosen conservatively such that modular wrap-around rarely happens). At $k=2$, the effect of modular clipping is now evident (the gap between DDGauss and Gaussian). With increasingly larger $k$ (larger $\gamma$), we incur more quantization errors (thus worse low bit-width performance) but less modular wrapping and can close the utility gap to Gaussian at high bit-widths.
\begin{figure*}[t]
    \centering
    \includegraphics[width=\linewidth]{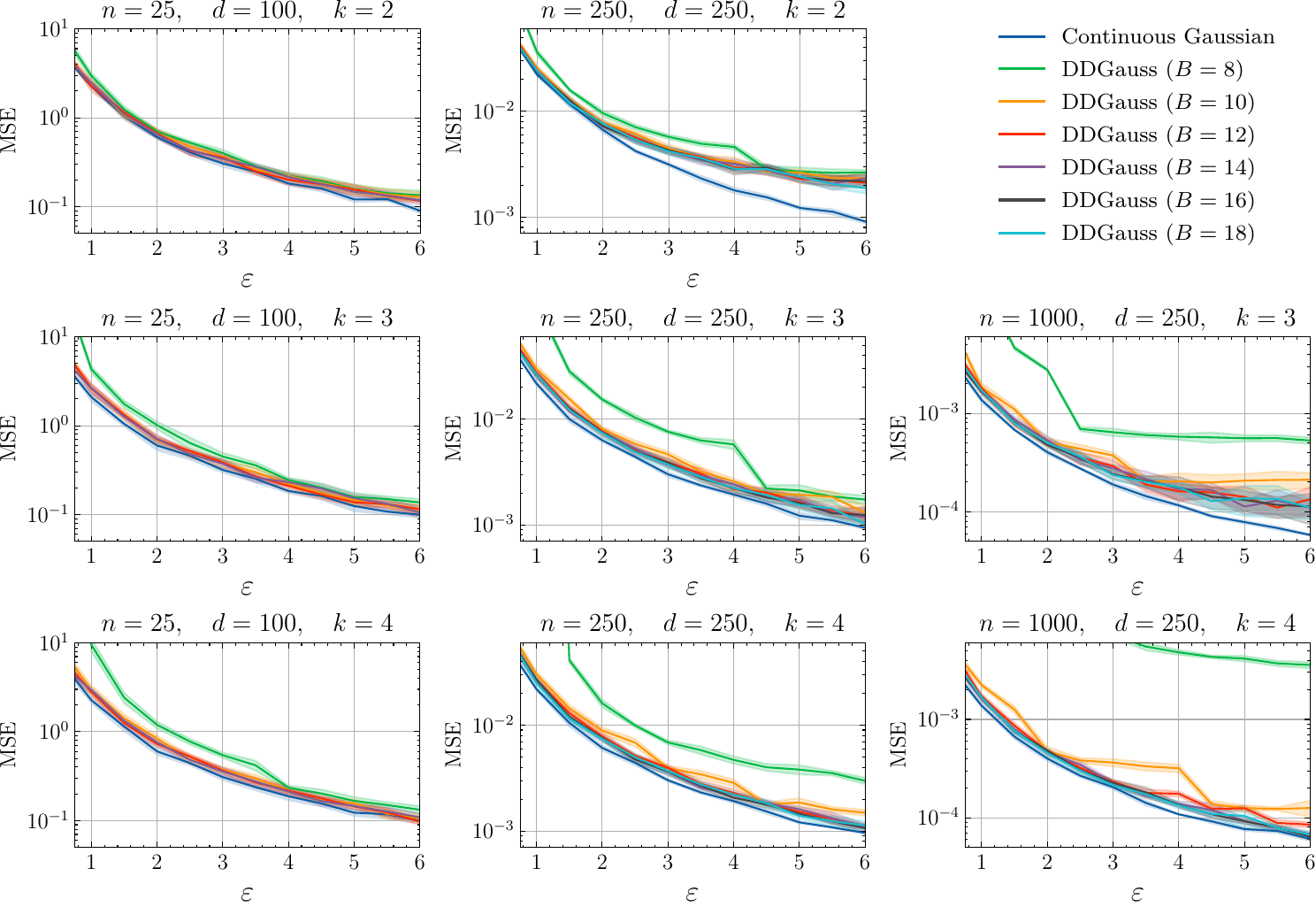}
    \vspace{-1em}
    \caption{Distributed mean estimation with the distributed discrete Gaussian, assuming an optimistic bound $\left\| \sum_i^n x_i \right\|_2 \le c \sqrt n$ as vectors are sampled uniformly randomly from a sphere. The first, second, and third row uses $k=2$, $k=3$, and $k=4$, respectively. $\delta=10^{-5}$. }
    \label{fig:dme-results-extended}
\end{figure*}

\subsection{Federated Learning}\label{sec:experiments-fl}

We evaluate on three public FL benchmarks: Federated EMNIST~\cite{cohen2017emnist, caldas2018leaf}, Stack Overflow Tag Prediction (SO-TP,~\cite{tffauthors2019}), and Stack Overflow Next Word Prediction (SO-NWP,~\cite{tffauthors2019}).

\subsubsection{Datasets}
Federated EMNIST is an image classification dataset containing 671,585 training hand-written digit/letter images over 62 classes grouped into $N=3400$ clients by their writer.
Stack Overflow is a large-scale text dataset based on the question answering site Stack Overflow. It contains over $10^8$ training sentences extracted from the site grouped by the $N = 342477$ users, and each sentence has associated metadata such as tags.
The task of SO-TP involves predicting the tags of a given sentence, while the task of SO-NWP involves predicting the next words given the preceding words in a sentence. For more details on the datasets and tasks, we refer the reader to~\cite{reddi2020adaptive}. 
We note that these datasets differ from those commonly used in related work (e.g. MNIST~\cite{lecun2010mnist} and CIFAR-10~\cite{cifar10}) in that they are substantially larger, more challenging, and involve \textit{user-level} (instead of example-level) DP with real-world client heterogeneity and label/size imbalance.
Obtaining a small $\varepsilon$ on EMNIST is also harder due to the relatively large sampling rate $q=n/N$ needed for stable convergence under noising.

\subsubsection{Models}
For EMNIST, We train a small convolutional net with two 3$\times$3 conv layers with 32/64 channels followed by two fully connected layers with 128/62 output units; a 2$\times$2 max pooling layer and two dropout layers with drop rate 0.25/0.5 are added after the first 3 trainable layers, respectively. The total number of parameters is $d=1018174$, which is slightly under $2^{20}$ to avoid excessive zero padding for the Walsh-Hadamard transform.
For SO-TP, we follow~\cite{reddi2020adaptive} and train a simple logistic regression model for tag prediction. The vocabulary size is limited to $10000$ for word tokens and 500 for tags, and each sentence is represented as a bag-of-words vector. The resulting model size is $d = 5000500$, which incurs a significant amount of zero padding. For SO-NWP, we use the LSTM architecture defined in~\cite{reddi2020adaptive} directly, which has a model size of $d = 4050748$ parameters (slightly under $2^{22}$).

\begin{figure}[!t]
    \centering
    \includegraphics[width=0.35\linewidth]{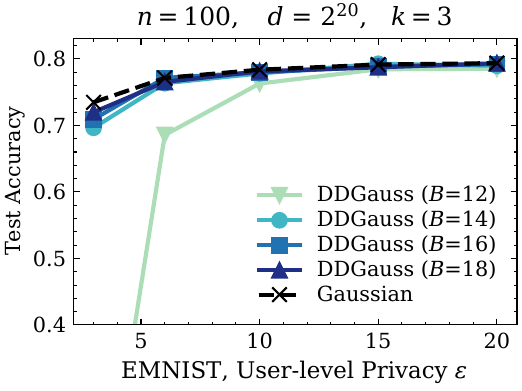}
    \caption{Summary of test accuracies (averaged over last 100 rounds) on EMNIST at $k=3$ across different values of $\varepsilon$ and $B$. $d$ is the padded model size. $\delta = 1/N$.}
    \label{fig:supp-emnist-summary}
    \vspace{1em}
    \includegraphics[width=\linewidth]{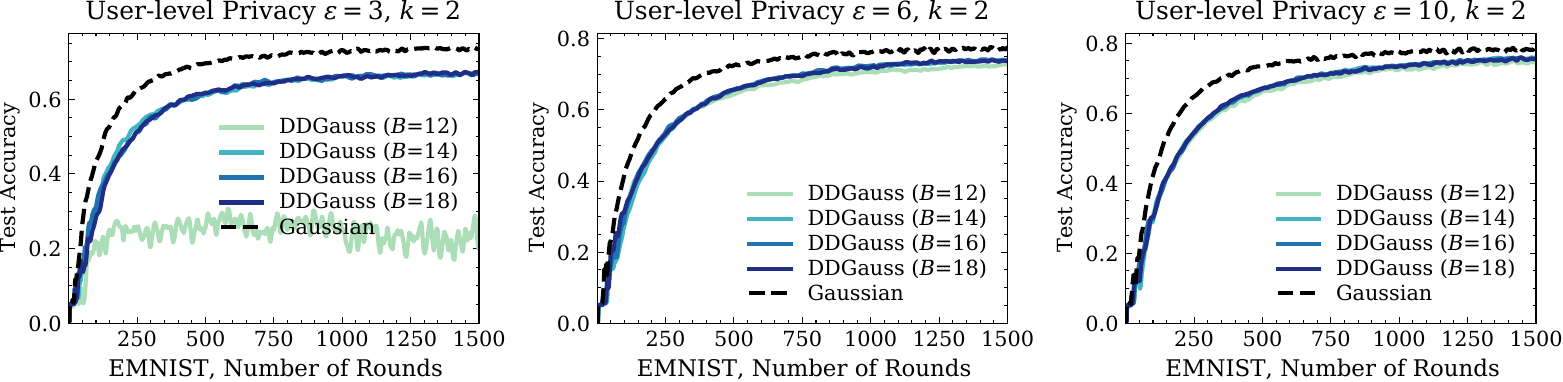} \\
    \vspace{1em}
    \includegraphics[width=\linewidth]{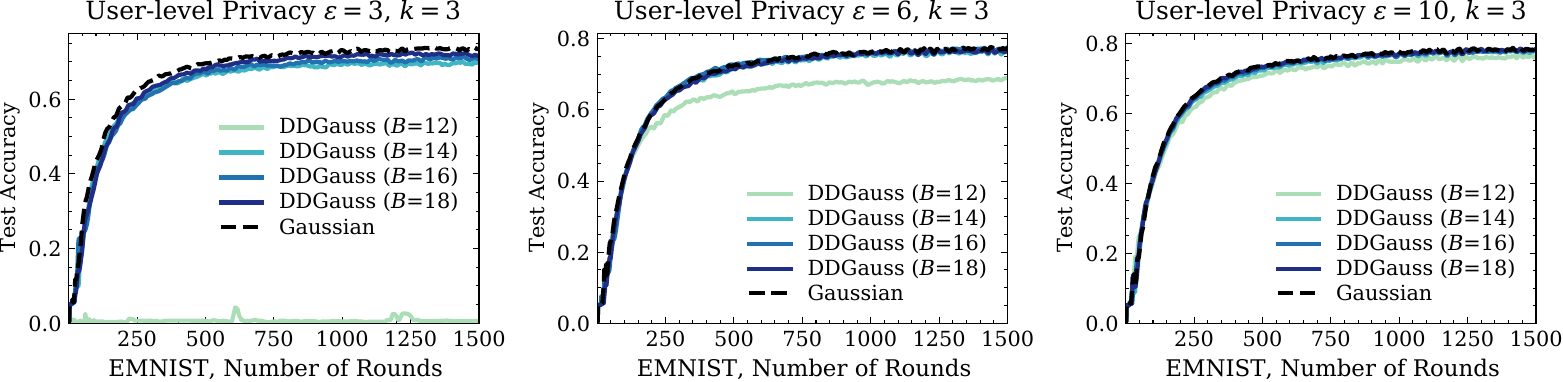} \\
    \vspace{1em}
    \includegraphics[width=\linewidth]{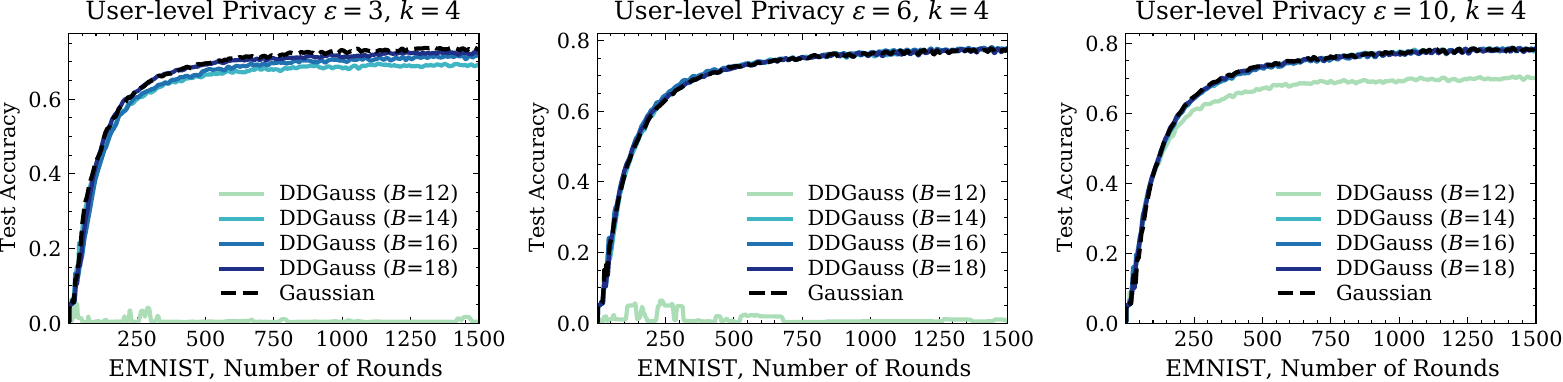}
    \vspace{-1em}
    \caption{Test accuracies during training for EMNIST across different $\varepsilon$, $B$, and $k$. $\delta = 1/N$.}
    \label{fig:supp-emnist-rounds}
\end{figure}

\subsubsection{Setup}

For all benchmarks, we used the standard dataset split provided by TensorFlow. For EMNIST, the dataset is split into training and test set and performance is reported on the test set. For Stack Overflow (SO-TP and SO-NWP), the dataset is split into training, validation, and test sets. Validation accuracies and test accuracies are reported on the validation and test sets respectively.
Note that using the dataset splits from TensorFlow is standard practice as in previous work (e.g.~\cite{reddi2020adaptive,mcmahan2018learning,al-shedivat2021federated}) and it allows our results to be comparable in similar settings. Note also that validation techniques such as $k$-fold validation can incur additional privacy costs.

We adopt most hyperparameters from previous work~\cite{reddi2020adaptive, andrew2019differentially, kairouz2021practical}.
For all tasks, we train with federated averaging with server momentum of 0.9~\cite{mcmahan2017communication,hsu2019measuring}. In each round, we uniformly sample $n=100$ clients for EMNIST and SO-NWP following~\cite{andrew2019differentially} and $n=60$ clients for SO-TP due to memory limit. We train 1 epoch over clients' local datasets. Each client's model updates are weighted uniformly (instead of by their number of samples) to maintain privacy. Clients are sampled without replacement within each round, and with replacement across rounds.
For EMNIST, SO-TP, and SO-NWP respectively, we set the number of rounds $T$ to 1500, 1500, and 1600, $c$ to 0.03, 2.0, and 0.3, client learning rate $\eta_c$ to 0.032, 316, and 0.5, and client batch size to 20, 100, and 16. Server LR $\eta_s$ is set to 1 for EMNIST and 0.56 for SO-TP; for SO-NWP, $\eta_s$ is selected from a small grid \{0.3, 1\} and the best performance (according to validation accuracy) is reported. Tuning is limited to $c$ (to tradeoff between the bias from clipping and the noise from privacy) and $\eta_s$ (to match the selected $c$ and $n$). For SO-NWP, we limit the max number of examples per client to 256.

The reported privacy guarantees $\varepsilon$ rely on privacy amplification via sampling \cite{kasiviswanathan2008ldp, bassily2014private, abadi2016deep}, which is necessary to obtain reasonable privacy-accuracy tradeoffs in differentially private deep learning. This assumes that the identities of the users sampled in every round are hidden from the adversary. This does not hold for the entity initiating connection with the clients (typically the server running the FL protocol) but is applicable to the analysts that have requested the model. 
We adopt the tight amplification bound from~\cite{mironov2019r} for the Gaussian baseline and use the generic upper bound from~\cite{zhu2019poission} for DDGauss (we do not explore a precise analysis in this work). The generic amplification upper bound could lead to more noise being added for DDGauss to achieve the same privacy as Gaussian.

\begin{figure}[!t]
    \centering
    \includegraphics[width=0.32\linewidth]{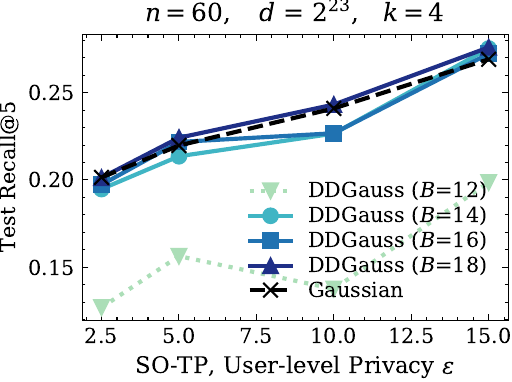}
    \includegraphics[width=0.32\linewidth]{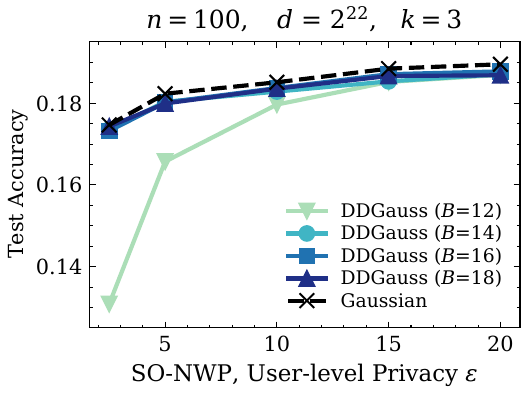}
    \includegraphics[width=0.32\linewidth]{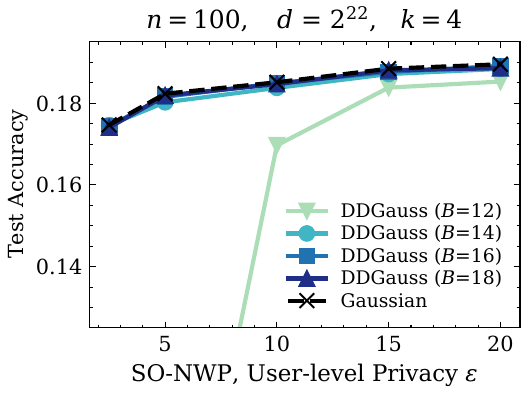}
    \caption{Summary of test performance on Stack Overflow. \textbf{Left:} Tag Prediction with logistic regression. Note that all runs of $B=12$ except at $\varepsilon=15$ did not converge. \textbf{Middle and Right:} Next Word Prediction with $k=3$ and $k=4$, respectively. $d$ is the padded model size. $\delta=1/N$ for SO-TP and $\delta = 10^{-6}$ for SO-NWP. }
    \label{fig:supp-so-summary}
    \vspace{1em}
    \centering
    \includegraphics[width=\linewidth]{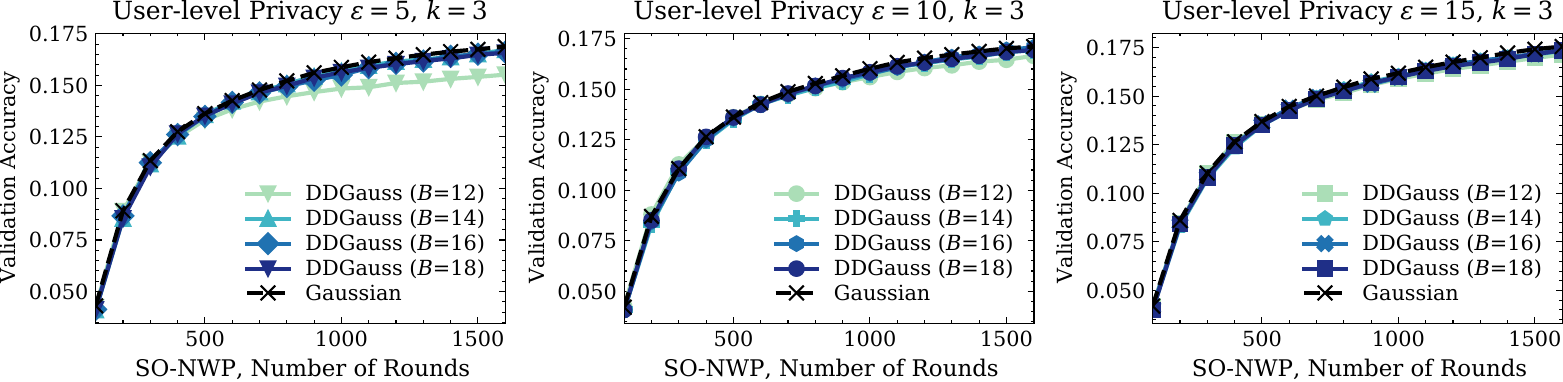} \\
    \vspace{1em}
    \includegraphics[width=\linewidth]{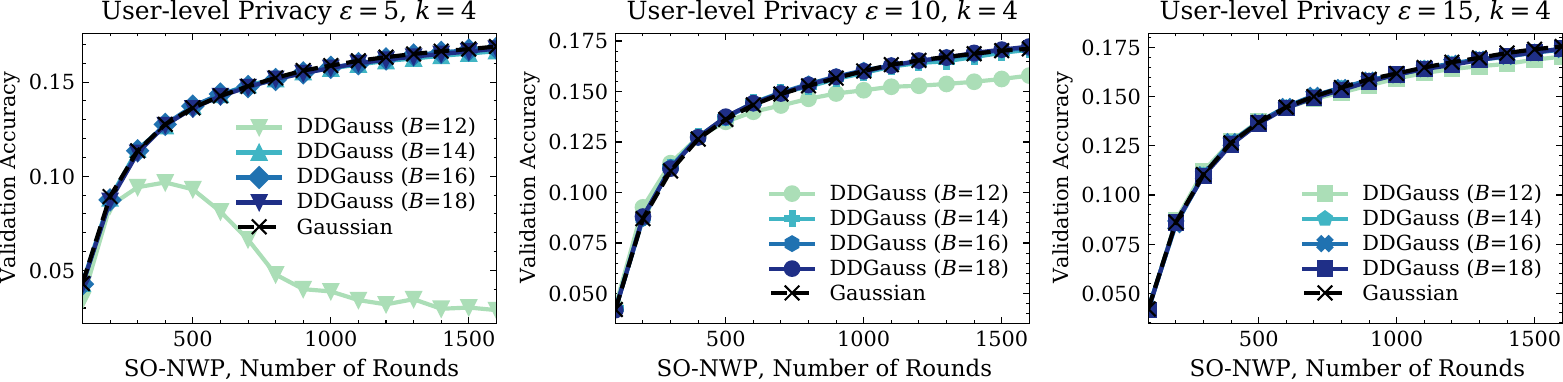}
    \vspace{-1em}
    \caption{Validation accuracies during training for SO-NWP across different $\varepsilon$, $B$, and $k$.}
    \label{fig:supp-sonwp-rounds}
\end{figure}

\subsubsection{Results}

For EMNIST, Figure~\ref{fig:supp-emnist-summary} summarizes the test accuracies across different values of $\varepsilon$ and $B$ for $k=3$, and Figure~\ref{fig:supp-emnist-rounds} shows the accuracies during training.
For Stack Overflow, Figure~\ref{fig:supp-so-summary} summarizes the test performance on SO-TP (recall@5) and SO-NWP (accuracy), and Figure~\ref{fig:supp-sonwp-rounds} shows the validation accuracies on SO-NWP. 

Overall, with more communication bits ($B$) and privacy budget ($\varepsilon$), DDGauss
achieves a better utility both relative to the Gaussian baseline and in absolute performance, and it can match the continuous Gaussian as long as $B$ is sufficient.

In particular, we can again observe the trade-off between quantization and modular clipping from Figure~\ref{fig:supp-emnist-rounds} and~\ref{fig:supp-so-summary}: a small $k$ can be sub-optimal for learning as the cost of modular wrap-around is more pronounced than quantization errors; using a larger $k$ allows DDGauss to match the Gaussian baseline at the expense of worse low bit-width performance (as $\gamma$ is larger).

Note also that for EMNIST (Figure~\ref{fig:supp-emnist-rounds}), there is a slight performance gap between Gaussian and DDGauss in the extreme setting with $\varepsilon=3$ and $k=4$. We believe this minor mismatch, on top of the errors from rounding and modular clipping, is due to the use of the generic upper bound for privacy amplification via subsampling as discussed earlier in this section.

\subsection{Additional Results}

\begin{figure}[t]
    \centering
    \includegraphics[width=0.5\linewidth]{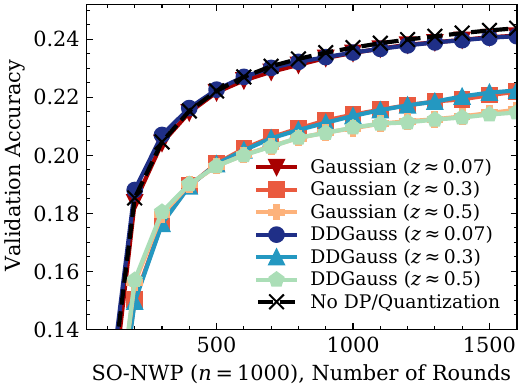}
    \vspace{-0.5em}
    \caption{Validation accuracies on SO-NWP (averaged every 100 rounds) with $n=1000$ and $B=18$. $z$ is the approximate noise multiplier.\textsuperscript{\ref{footnote:noise-mult}} The no privacy/quantization baseline and $z\approx 0.07$ runs uses $\eta_s =3$ while others use $\eta_s = 1$.
    }
    \label{fig:sonwp-n1000}
\end{figure}
\begin{figure}[t]
    \centering
    \includegraphics[width=\linewidth]{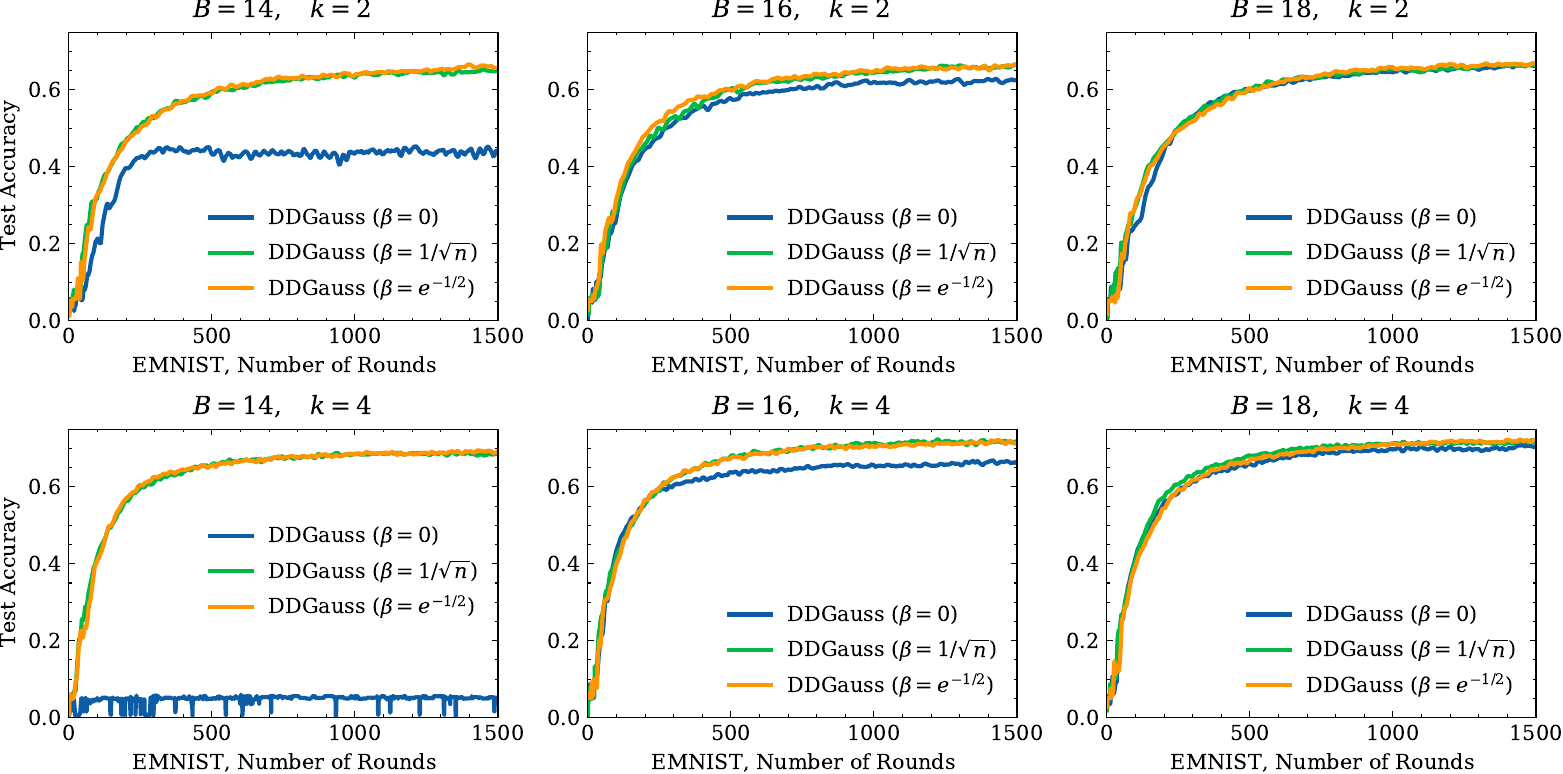}
    \vspace{-1em}
    \caption{Effects of $\beta$ on Federated EMNIST. The number of clients per round is $n=100$ and the user-level privacy budget is fixed at $\varepsilon=3$. $\delta = 1/N$.  }
    \label{fig:emnist-beta}
\end{figure}
\begin{figure}[!htb]
    \centering
    \includegraphics[width=0.5\linewidth]{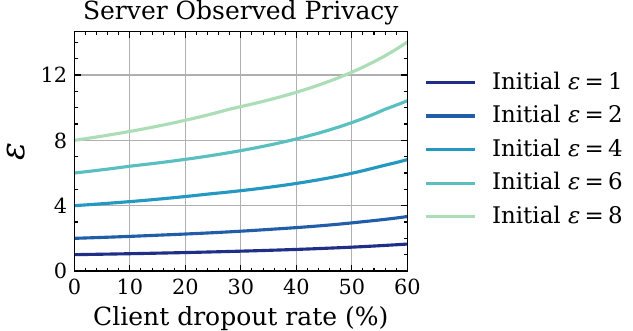}
    \caption{Effects of client dropouts on the server observed privacy. $\delta = 10^{-5}$.}
    \label{fig:eps-dropout}
\end{figure}

\paragraph{Large-Scale Training} \label{sec:supp-large-scale-training}

We additionally consider scaling up the SO-NWP experiments to $n=1000$ clients per round (similar to production settings described in~\cite{hard2018federated, mcmahan2018learning,ramaswamy2019federated}), and we show the validation accuracies during training across different noise multipliers\footnote{\label{footnote:noise-mult}$z=\widehat\sigma/c$ where $\widehat\sigma$ is the equivalent central noise standard deviation (i.e. $\sqrt{n}\sigma$ for DDGauss). The values of $z$ are aligned on privacy budgets and thus $z$ is in fact slightly larger for DDGauss compared to Gaussian due to effects of rounding, generic amplification, etc.} in Figure~\ref{fig:sonwp-n1000}.
We set $c=1$ and $\eta_s=1$ for $z\approx0.3$ and $z\approx0.5$, and we set $\eta_s=3$ otherwise. 
$z\approx 0.07$ gives a target test accuracy of around 25.2\% (e.g. a utility-first approach to limit performance degradation from DP~\cite{kairouz2021practical}) while $z\approx 0.5$ and $z\approx 0.3$ give $\varepsilon$ of around 10 and 234 respectively. The results bear significant practical relevance as they indicate that as long as DDGauss is parameterized properly, it can perform as good as the continuous Gaussian in real-world settings (with large $n$, large model size $d$, and natural client heterogeneity from Stack Overflow).

\paragraph{Effects of $\beta$}
Recall from Section~\ref{sec:rr} that the hyperparameter $\beta$ controls the growth of the client vector norm from conditional randomized rounding. Here, we are interested to know how $\beta$ and the bias and variance it introduces influence the communication-utility trade-off in practice. Figure~\ref{fig:emnist-beta} shows the results on Federated EMNIST with $\beta \in \{0, \frac{1}{\sqrt{n}}, e^{-1/2}\}$ across $B\in \{14, 16, 18\}$ and $k \in \{2, 4\}$ with user-level privacy budget fixed at $\varepsilon = 3$; other parameters follow those described earlier.\footnote{$\beta = 0$ leads to unconditional rounding, in which case we use the worst case bound $\Delta_{2} \leq \|x\|_{2}+\gamma \sqrt{d}$.}
We note that when the communication budget is tight (i.e. large $k$ and small $B$, where the ``room'' for larger norm and noise variance is limited), the bounded norm growth from conditional rounding can be pivotal to model learning and convergence. When the communication budget is sufficient (i.e. small $k$ and large $B$, where we can afford unconditional rounding), the bias introduced by $\beta > 0$ have insignificant impact on the model utility (e.g. $\beta = e^{-1/2} \approx 0.607$ and $\beta = 1/\sqrt n = 0.1$ give similar performance and convergence speed).

\paragraph{Privacy Degradation from Client Dropouts}
Figure~\ref{fig:eps-dropout} shows the privacy degradation as observed by the server if a certain percentage of the clients drops out during aggregation (thus there would be missing local noise shares). Note that for the external analyst, the server can always add the missing shares of noise onto the aggregate to prevent this degradation. Note also that the values of the parameters ($\gamma, \beta, B, c, d, k, n, T$) does not affect the degradation as they influence each other to arrive at the same initial $\varepsilon$. The sampling rate $q$ is also fixed at 1.0 as subsampling does not apply from the server's perspective. Results indicate that the privacy guarantees degrade gracefully as clients drop out.

\section{Concluding Remarks} \label{sec:conclusion}
 We have presented an complete end-to-end protocol for federated learning with distributed DP and secure aggregation. Our solution relies on efficiently flattening and discretizing the client model updates before adding discrete Gaussian noise and applying secure aggregation. A significant advantage of this approach is that it allows an untrusted server to perform complex learning tasks on decentralized and privacy-sensitive data while achieving the accuracy of a trusted server. Our theoretical guarantees highlight the complex tension between communication, privacy, and accuracy. Our experimental results demonstrate that our solution is essentially able to match the accuracy of central differential privacy with 16 or fewer bits of precision per value. 
 
 Several questions remain to be addressed, including (a) tightening the generic RDP amplification via sampling results or conducting a precise analysis of the subsampled distributed discrete Gaussian mechanism, (b) exploring the use of a discrete Fourier transform or other methods instead of the Walsh-Hadamard transform to avoid having to pad by (up to) $d-1$ zeros, (c) developing private self-tuning algorithms that learn how to optimally set the parameters of the algorithm on the fly, and (d) proving a lower bound on $m$ that either confirms that the distributed discrete Gaussian's $m \ge \tilde{O}\left( n  + \sqrt{\frac{\varepsilon^2 n^3}{d}} + \frac{\sqrt{d}}{\varepsilon}\right)$ is order optimal or suggests the existence of a better mechanism.

\section{Acknowledgments} \label{sec:acknolwedgments}
We thank Naman Agarwal and Kallista Bonawitz for helpful discussions and comments on drafts of this paper. We thank Andrea La Mantia for pointing out an error in one of our calculations.

\printbibliography

\end{document}